\documentclass[12pt,a4paper]{article}

\usepackage[margin=1in]{geometry}
\usepackage{amsmath,amssymb,amsthm}
\usepackage{graphicx}
\usepackage{hyperref}
\usepackage{natbib}
\usepackage{booktabs} 
\usepackage{tikz}
\usetikzlibrary{positioning, arrows.meta, calc, shapes.geometric} 
\usepackage{caption}          % 支持 \captionsetup
\usepackage{algorithm}        % 定义 algorithm 计数器与浮动环境
\usepackage{algpseudocode}    % 支持 \State, \For, \While, \Comment, \Statex 等算法语法
\usepackage{array}
\usepackage{tabularx}  % 支持 tabularx 环境和 X 列宽自适应
\usepackage{booktabs}  % 支持 \toprule, \midrule, \bottomrule, \addlinespace
\usepackage{amsmath}
\usepackage{amssymb}
\usepackage{graphicx} 
\newtheorem{theorem}{Theorem}[section]
\newtheorem{proposition}{Proposition}[section]

\newtheorem{definition}{Definition}[section]
\newtheorem{remark}{Remark}[section]
\usepackage{booktabs}
\usepackage{array}
\usepackage{wrapfig}
\usepackage{float}
\usepackage{multirow}
\usepackage{pifont}
\DeclareRobustCommand{\cmark}{\text{\ding{52}}} % ✓
\DeclareRobustCommand{\xmark}{\text{\ding{55}}} % ✗

\hypersetup{
    hidelinks,
    pdfauthor={Anonymous},
    pdftitle={Your Paper Title},
    pdfsubject={Submission},
    pdfkeywords={keyword1, keyword2}
}

\title{\textbf{OVA-IB: One vs All Information Bottleneck for Multi-Modal Alignment}}
\author{
\begin{tabular}{ccc}
Tianchao Li$^{1}$ 
& Shujian Yu$^{2,3}$ 
& Xinrui Zu$^{2}$ \\
Zhaolong Wei$^{6}$ 
& Jeremy Gummeson$^{6}$ 
& Jack C.P. Cheng$^{1}$ \\
\multicolumn{3}{c}{Robert Jenssen$^{3,4,5}$} \\
\end{tabular}
\\[1.2em]
$^{1}$Hong Kong University of Science and Technology, Hong Kong \\
$^{2}$Vrije Universiteit Amsterdam, Netherlands\\
$^{3}$UiT -- The Arctic University of Norway, Norway \\
$^{4}$University of Copenhagen, Denmark \\
$^{5}$Norwegian Computing Center, Norway \\
$^{6}$University of Massachusetts Amherst, USA
}
\date{}

\begin{document}
\maketitle
\begin{abstract}
\noindent Contrastive learning is effective for aligning paired views or modalities, but alignment beyond two modalities remains non-trivial and comparatively underexplored. Pairwise CLIP-style losses decompose multi-modal alignment into independent two-way comparisons and therefore do not explicitly model higher-order dependencies among multiple modalities. Recent beyond-pairwise objectives approach this problem from statistical or geometric perspectives, but arbitrary-modality alignment still lacks a principled criterion for defining what each modality should preserve and compress relative to the others. We revisit arbitrary-modality alignment through the Information Bottleneck principle. In multi-modal learning, sufficiency should preserve information predictable from the remaining modalities, while minimality should compress modality-specific information not supported by them. This naturally leads to a One-vs-All view, where each modality is characterized with respect to the remaining modalities. We propose OVA-IB, an Information Bottleneck framework for arbitrary-modality alignment. OVA-IB optimizes a tractable One-vs-All contrastive lower bound for sufficiency connected to a Dual Total Correlation-style objective, uses a parameter-free geometry-aware projection score, and derives a tractable upper-bound regularizer for minimality by bounding each representation's dependence on its own input with representation distributions induced by the remaining modalities. Experiments on classification, regression, modality-agnostic evaluation, and cross-modal retrieval benchmarks demonstrate strong and robust performance.
\end{abstract}

\section{Introduction}
\label{sec:intro}
Contrastive Learning (CL) has become a central paradigm for self-supervised representation learning, as it enables models to learn transferable representations without relying on extensive human annotations. In the single-modality setting, different views of the same samples are typically generated through data augmentation, and the model is trained to bring positive views closer while separating representations from different samples. As a widely used objective for this purpose, InfoNCE \cite{oord2019representationlearningcontrastivepredictive, pmlr-v119-chen20j}  maximizes agreement between paired views and has been interpreted from both information-theoretic and geometric perspectives. From the information-theoretic perspective, InfoNCE maximizes a lower bound on mutual information between positive views \cite{10.5555/3495724.3496212}; from the geometric perspective, contrastive learning simultaneously promotes alignment between positive pairs and uniformity of representations on the hypersphere \cite{pmlr-v119-wang20k}. These two perspectives provide the foundation for understanding contrastive learning as both an information-preserving and geometry-shaping representation learning framework.

For two-modality representation learning, paired modalities provide natural positive views for contrastive alignment. A prominent example is CLIP \cite{pmlr-v139-radford21a}, which applies a symmetric InfoNCE objective to image--text pairs and learns a shared embedding space with strong transfer performance. However, this formulation is intrinsically pairwise. For three or more modalities, a common extension is to sum CLIP-style losses over all modality pairs. While simple and scalable, this strategy decomposes multi-modal alignment into independent two-modality comparisons and therefore fails to explicitly capture higher-order dependencies jointly supported by multiple modalities. This motivates alignment objectives that move beyond pairwise agreement. Recent beyond-pairwise methods address this limitation from complementary directions. Symile~\cite{saporta2024contrasting} models higher-order statistical dependence among an arbitrary number of modalities, while Gram~\cite{cicchetti2025gramian} and TRIANGLE~\cite{cicchetti2025a} introduce geometry-aware objectives in the shared representation space. These works provide important statistical and geometric formulations of beyond-pairwise alignment. However, alignment beyond two modalities remains comparatively underexplored, and it is still unclear how to define what each modality representation should preserve and compress when multiple remaining modalities are available.

To instantiate this principle, we propose OVA-IB, a One-vs-All Information Bottleneck framework for multi-modal alignment with an arbitrary number of modalities. The key idea is to define both sufficiency and minimality in a modality-wise manner: for each modality, the remaining modalities serve as the reference for determining what information should be preserved and what modality-specific variation should be compressed. Specifically, the sufficiency term encourages each modality representation to retain information predictable from the remaining modalities. We derive a tractable One-vs-All contrastive lower bound for this objective and show that it is connected to a Dual Total Correlation \cite{5392532} (DTC)-style dependence measure. To make the alignment geometry-aware, we score each modality by its projection onto the subspace spanned by the remaining modality embeddings. For minimality, we derive a tractable upper-bound regularizer that controls each representation's dependence on its own input using representation distributions induced by the remaining modalities.

Our contributions are summarized as follows:
\begin{itemize}
\item We propose OVA-IB, a One-vs-All Information Bottleneck framework for multi-modal alignment with an arbitrary number of modalities.
\item We derive a DTC-style sufficiency objective that aligns each modality with the remaining modalities as complementary evidence, and introduce a geometry-aware projection score that aligns each embedding with the subspace spanned by the remaining modalities.
\item We derive a tractable minimality regularizer that suppresses modality-specific nuisance information by bounding each representation's dependence on its own input using distributions induced by the remaining modalities.
\item Experiments on classification, regression, modality-agnostic evaluation, and cross-modal retrieval benchmarks demonstrate that OVA-IB achieves strong and robust performance.
\end{itemize}

\section{Related Work}
\label{sec:related works}
\subsection{Contrastive Learning for Multi-Modal Representation Learning}
Contrastive learning (CL) has been widely adopted for self-supervised representation learning by aligning positive views and separating negative samples \cite{oord2019representationlearningcontrastivepredictive,pmlr-v119-chen20j,10.5555/3495724.3496212,rusak2025infonce,pmlr-v119-wang20k,wu2025pca,pmlr-v139-zimmermann21a,luthra2025selfsupervised}. In multi-modal learning, different modalities of the same sample naturally define positive pairs, enabling contrastive objectives to learn shared information across modalities. CLIP \cite{pmlr-v139-radford21a} applies a symmetric InfoNCE to image-text pairs and has become a standard paradigm for two-modality alignment:
\begin{align}   
\mathcal{L}^{(1\rightarrow2)}_{\text{InfoNCE}}(\theta) &= -\frac{1}{N}\sum_{i=1}^N\log \frac{\exp\left( \text{sim}(\mathbf{z}^{(1)}_i, \mathbf{z}^{(2)}_i)/\tau \right)}{\sum_{k=1, k\ne i}^{N}  \exp\left( \text{sim}(\mathbf{z}^{(1)}_i, \mathbf{z}^{(2)}_k)/\tau \right)}, \\
    \mathcal{L}^{(1,2)}_{\text{CLIP}}(\theta)&= \frac{1}{2}(\mathcal{L}^{(1\rightarrow2)}_{\text{InfoNCE}}(\theta)+\mathcal{L}^{(2\rightarrow 1)}_{\text{InfoNCE}}(\theta)),
\end{align}
where $\mathbf{z}_i^{(m)}$ represents the representation of the $m$-th modality of the $i$-th sample, $\tau$ is a temperature parameter, and $\theta$ represents all the learnable parameters. For more than two modalities, a common extension is to apply CLIP-style objectives over all modality pairs \cite{tian2020contrastive, 10.5555/3540261.3542114, Girdhar2023ImageBindOE, chen2023vast, 10.5555/3495724.3495727, 9709934, 10721284, wang2025omnibind, Cho2025SynergyCLIPEC}.  In the simplest case, for three modalities, the pairwise CLIP loss corresponds to
\begin{align}
    \mathcal{L}^{(1,2, 3)}_{\text{CLIP}}(\theta) = \mathcal{L}^{(1,2)}_{\text{CLIP}}(\theta) + \mathcal{L}^{(1,3)}_{\text{CLIP}}(\theta) + \mathcal{L}^{(2,3)}_{\text{CLIP}}(\theta).
\end{align}
It reduces multi-modal alignment to a collection of independent pairwise comparisons and therefore does not explicitly capture higher-order dependencies among multiple modalities.

Compared with the extensive literature on pairwise multi-modal contrastive learning, alignment objectives for more than two modalities remain comparatively underexplored, with only a small number of recent works explicitly addressing this setting. Symile~\cite{saporta2024contrasting} extends contrastive learning by modeling higher-order statistical dependence among arbitrary modalities, while Gram~\cite{cicchetti2025gramian} and TRIANGLE~\cite{cicchetti2025a} introduce geometry-aware objectives based on the volume or area formed by the same-sample modality embeddings. These works provide important initial steps toward beyond-pairwise alignment from statistical and geometric perspectives. Complementary to them, we revisit arbitrary-modality alignment from the Information Bottleneck \cite{tishby2000informationbottleneckmethod} perspective and ask how sufficiency and minimality should be defined for each modality when multiple remaining modalities are available.

\subsection{Information Bottleneck for Multi-Modal Representation Learning}

Information Bottleneck (IB) \cite{tishby2000informationbottleneckmethod} provides an information-theoretic framework for learning representations that are sufficient for a target variable while remaining minimal with respect to the input. Given an input $\mathbf{x}$, a target $\mathbf{y}$, and a representation $\mathbf{t}$, the standard IB objective is formulated as
\begin{equation}
\max_\mathbf{t} \ I(\mathbf{y};\mathbf{t}) - \beta I(\mathbf{x};\mathbf{t}),
\end{equation}
where \(I(\cdot;\cdot)\) denotes mutual information and \(\beta>0\) controls the strength of compression. The first term encourages \(\mathbf{t}\) to preserve information relevant to \(\mathbf{y}\), while the second term discourages \(\mathbf{t}\) from retaining unnecessary information from \(\mathbf{x}\). This sufficiency--minimality trade-off provides a natural principle for multi-modal representation learning, where the goal is to preserve shared semantic information while suppressing modality-specific nuisance factors \cite{wu2025learning, 10.1109/TMM.2022.3171679,e28040445}.

The general idea of IB has recently been applied for two-modality alignment \cite{almudevar2025aligning}. Given two modalities ($\mathbf{x}^{(1)}, \mathbf{x}^{(2)}$) with their representations ($\mathbf{z}^{(1)}, \mathbf{z}^{(2)}$), the sufficiency for $\mathbf{x}^{(1)}$ can be approximated through cross-modal dependence, while the minimality can be upper-bounded through the dependence between $\mathbf{z}^{(1)}$ and its own input $\mathbf{x}^{(1)}$. This yields the IB-style objective:
\begin{equation}
\label{eq:two_modality}
\max_{\mathbf{z}^{(1)}, \mathbf{z}^{(2)}} \ I(\mathbf{z}^{(1)};\mathbf{z}^{(2)}) - \frac{\beta}{2} (I(\mathbf{z}^{(1)};\mathbf{x}^{(1)})+I(\mathbf{z}^{(2)};\mathbf{x}^{(2)})).
\end{equation}
This formulation is elegant. However, extending Eq. (\ref{eq:two_modality}) for an arbitrary number of modalities is a non-trivial task. This is because each modality is compared against multiple remaining modalities, and there is no unique paired counterpart as in the two-modality case.  Similarly, minimality specifies which information should be discarded relative to the evidence provided by the remaining modalities.

\section{One vs All Information Bottleneck for Multi-Modal Alignment}
\label{sec:methodology}

\subsection{Problem Setup and Multi-Modal IB Principle}

Let $\mathcal{X}^{(1)}, \dots, \mathcal{X}^{(M)}$ denote $M$ distinct modalities. We consider modality-specific encoders $f^{(m)}: \mathcal{X}^{(m)} \to \mathbb{R}^d$ and projection heads $g^{(m)}: \mathbb{R}^d \to \mathbb{R}^d$ for $m=1,\dots,M$. Given an input $\mathbf{x}^{(m)} \in \mathcal{X}^{(m)}$, the encoder produces a representation, which is mapped to the embedding ${\mathbf{z}}^{(m)} = g^{(m)}(f^{(m)}(\mathbf{x}^{(m)}))$ in the shared representation space. The primary objective of multi-modal representation learning is to optimize the encoders $f^{(m)}$ by aligning the projected embeddings $\{{\mathbf{z}}^{(m)}\}_{m=1}^M$ at the sample level. For simplicity, we denote by $[M] \backslash (m)$ the set of all indices except $m$, $\mathbf{z}^{[M]}=(\mathbf{z}^{(1)},\dots, \mathbf{z}^{(M)})$ represents representations of all $M$ modalities, and $\mathbf{z}^{[M]\backslash(m)}=(\mathbf{z}^{(1)},\dots, \mathbf{z}^{(m-1)}, \mathbf{z}^{(m+1)},\dots, \mathbf{z}^{(M)})$ represents representations of all $M$ modalities without $\mathbf{z}^{(m)}$. 

Our goal is to derive a multi-modal alignment objective from the Information Bottleneck principle. Following \cite{almudevar2025aligning}, we assume that each modality \(\mathbf{x}^{(m)}\) consists of an essence $\mathbf{y}^{(m)}$ and nuisance $\mathbf{n}^{(m)}$, where $\mathbf{y}^{(m)}$ contains all the relevant information in $\mathbf{x}^{(m)}$, while $\mathbf{n}^{(m)}$ represents the modality-specific noise in $\mathbf{x}^{(m)}$. Conceptually, a representation $\mathbf{z}^{(m)}$ is sufficient if it preserves information about $\mathbf{y}^{(m)}$, and minimal if it discards information about $\mathbf{n}^{(m)}$. This leads to the IB objective:
\begin{align}
\label{eq:concept_ib}
    \max_{\mathbf{z}^{(1)},\dots, \mathbf{z}^{(M)}} \frac{1}{M}\sum_{m=1}^M \left(I(\mathbf{z}^{(m)}; \mathbf{y}^{(m)}) - \beta I(\mathbf{z}^{(m)}; \mathbf{n}^{(m)})\right).
\end{align} 
The superscript \(m\) only indicates that the essence and nuisance components may differ across modalities; it does not imply that Eq.~(\ref{eq:concept_ib}) defines separate single-modality IB objectives.

\subsection{Multi-Modal Sufficiency}
We first derive the sufficiency objective for the $m$-th modality. From the IB objective in Eq. (\ref{eq:concept_ib}), sufficiency requires maximizing $I(\mathbf{z}^{(m)}; \mathbf{y}^{(m)})$.  In the two-modality setting, the essence of one modality can be characterized through its paired modality. In the arbitrary-modality setting, however, each modality is accompanied by multiple remaining modalities, which together provide complementary context for determining what should be preserved in \(\mathbf x^{(m)}\). Therefore, for each modality, we characterize \(\mathbf y^{(m)}\) with respect to the remaining modalities \(\mathbf x^{[M]\backslash(m)}\) and define the essence in Definition \ref{ass:markov chain}. Intuitively, $\mathbf{x}^{[M]\backslash(m)}$ provides the cross-modal context that determines which information in $\mathbf{x}^{(m)}$ should be preserved by a sufficient representation. 

\begin{definition}
\label{ass:markov chain}
    For each modality \(\mathbf{x}^{(m)}\), the essence \(\mathbf y^{(m)}\) is characterized with respect to the remaining modalities \(\mathbf x^{[M]\backslash(m)}\)
    through the following Markov chains:
    \[
\mathbf y^{(m)} \leftrightarrow \mathbf x^{[M]\backslash (m)} \leftrightarrow \mathbf x^{(m)},
\qquad
\mathbf x^{[M]\backslash (m)} \leftrightarrow \mathbf y^{(m)} \leftrightarrow \mathbf x^{(m)}.
\]
\end{definition}
The first relation states that, conditioned on the remaining modalities, the modality does not provide additional information about its cross-modal essence. The second relation states that the essence \(\mathbf y^{(m)}\) contains the information in \(\mathbf x^{(m)}\) that is supported by the remaining modalities.

\begin{theorem}
\label{thm:m_sufficency_objective}
Under Definition~\ref{ass:markov chain}, suppose the representation \(\mathbf z^{(m)}\) is generated from the modality \(\mathbf x^{(m)}\), i.e., the encoder does not directly access \(\mathbf y^{(m)}\) or \(\mathbf x^{[M]\backslash(m)}\). Then the sufficiency term satisfies:
\begin{align}
\label{eq:1step_sufficiency}
I(\mathbf{z}^{(m)};\mathbf{y}^{(m)})=I(\mathbf{z}^{(m)};\mathbf{x}^{[M]\backslash(m)}).
\end{align}
\end{theorem}
\begin{proof}
    All proofs can be found in Appendix \ref{sec:proof of section}.
\end{proof}
Theorem~\ref{thm:m_sufficency_objective} shows that maximizing the sufficiency term \(I(\mathbf z^{(m)};\mathbf y^{(m)})\) can be reduced to maximizing the dependence between \(\mathbf z^{(m)}\) and \(\mathbf x^{[M]\backslash(m)}\). Since \(\mathbf z^{[M]\backslash (m)}\) is obtained from \(\mathbf x^{[M]\backslash (m)}\) through the encoders and projection heads of the remaining modalities, Data Processing Inequality gives $I(\mathbf z^{(m)};\mathbf x^{[M]\backslash (m)}) \geq I(\mathbf z^{(m)};\mathbf z^{[M]\backslash (m)})$. Therefore, we optimize the observable sufficiency objective
\begin{align}
\label{eq:m_sufficiency}
    \max_{\mathbf{z}^{(1)},\dots, \mathbf{z}^{(M)}} \frac{1}{M}\sum_{m=1}^M I(\mathbf{z}^{(m)};\mathbf{z}^{[M]\backslash(m)}).
\end{align}
This objective aligns one modality with all remaining modalities, revealing a One-vs-All paradigm from the sufficiency term. Computing exactly $I(\mathbf{z}^{(m)};\mathbf{z}^{[M]\backslash(m)})$ requires integrating over the representation spaces, which is generally intractable. To address this issue, we introduce a projection $t:(\mathbb{R}^d)^{(M-1)}\rightarrow \mathbb{R}^d$ into InfoNCE such that we can maximize the lower bound of such a multivariate mutual information term. Accordingly, we define the InfoNCE objective for the $m$-th modality with $N$ samples as:
\begin{align}
\label{eq:m_infonce}
    \mathcal{L}^{(m)}_{\text{InfoNCE}}(\theta) = - \frac{1}{N}\sum_{n=1}^N\log \frac{\exp(s({\mathbf{z}}_n^{(m)}, t({\mathbf{z}}_n^{^{[M]\backslash(m)}}))/\tau)}{\sum^N_{n'=1, n'\ne n}\exp(s({\mathbf{z}}_n^{(m)}, {t(\mathbf{z}}_{n'}^{^{[M]\backslash(m)}}))/\tau)},
\end{align}
where the score function $s$ outputs the similarity between one modality ${\mathbf{z}}^{(m)}$ and the output of $t$ with $M-1$ modalities ${\mathbf{z}}_n^{^{[M]\backslash(m)}}$, and cosine similarity is a commonly used choice for $s$. As Theorem \ref{thm:infonce_sufficiency} demonstrates, we can minimize the proposed InfoNCE variant from Eq. (\ref{eq:m_infonce}) to achieve the sufficiency of the $m$-th modality by maximizing the lower bound of Eq. (\ref{eq:m_sufficiency}). 
\begin{theorem}
    \label{thm:infonce_sufficiency}
    Minimizing the loss function in Eq. (\ref{eq:m_infonce}) corresponds to maximizing the lower bound of $I(\mathbf{z}^{(m)};\mathbf{z}^{[M]\backslash(m)})$.
\end{theorem}
A straightforward choice for the projection $t$ is to implement an MLP to map the concatenated ${\mathbf{z}}^{[M]\backslash(m)}$ from $\mathbb{R}^{(M-1)d}$ to $\mathbb{R}^d$, and then compute the cosine similarity. First, the MLP introduces many learnable parameters. Even if the parameter update is not counted, its computational complexity is $\mathcal{O}(d^2M)$. The MLP also neglects the geometric relation among modalities in the shared representation space and performs dimensionality reduction for cosine similarity.

\begin{wrapfigure}[15]{r}{0.45\columnwidth}
\vspace{-\dimexpr\baselineskip+\abovecaptionskip}
  \centering
 \includegraphics[width=\linewidth]{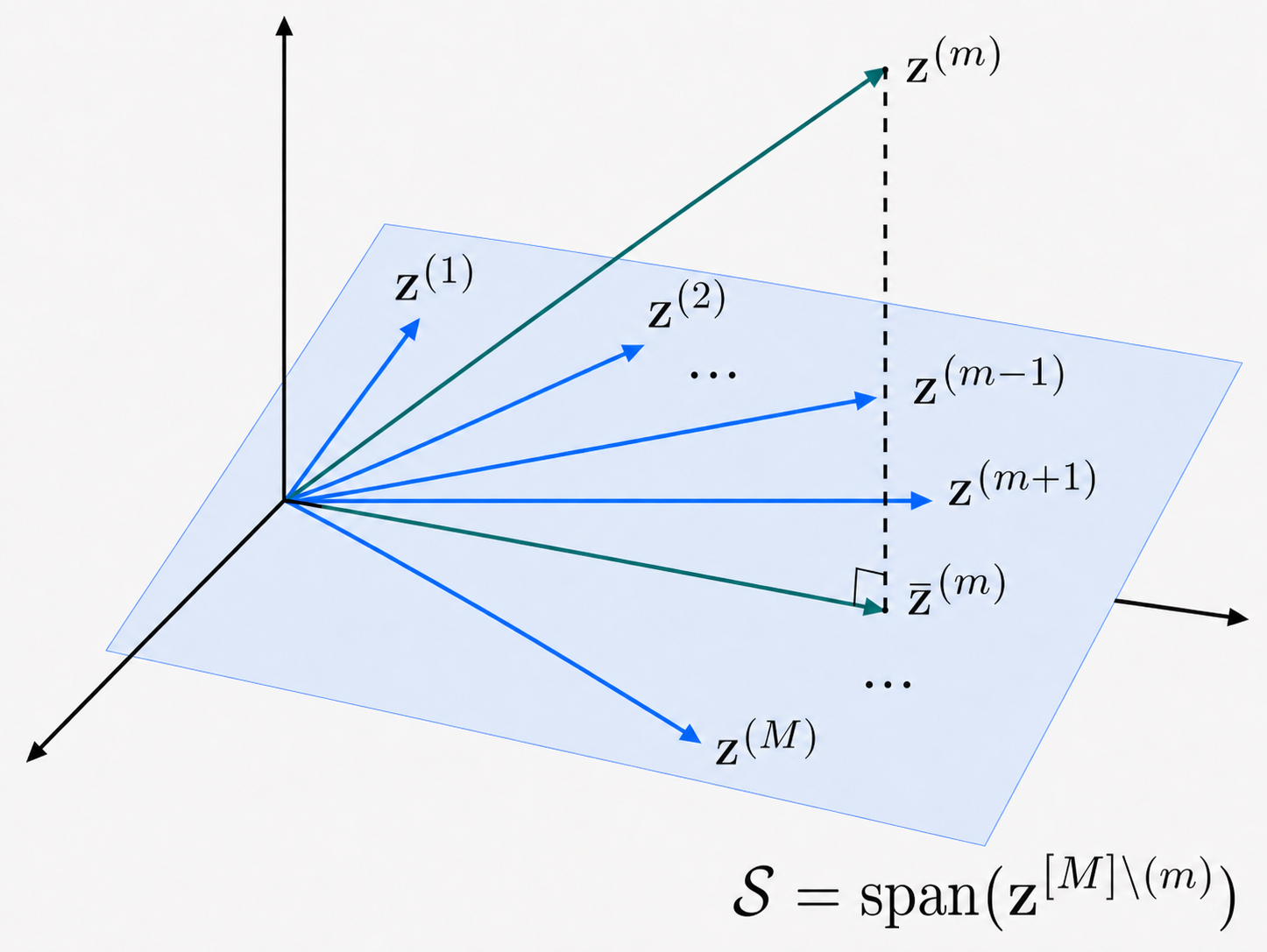}
 \caption{
We align ${\mathbf{z}}^{(m)}$ with its projection $\bar{\mathbf{z}}^{(m)}$ onto $\mathcal{S}$,
encouraging each modality to preserve information predictable from the remaining
modalities.
}
  \label{fig:halfcol}
\end{wrapfigure}
If the remaining modalities jointly provide sufficient relevant information for the same sample, the cross-modal information of ${\mathbf{z}}^{(m)}$ should be well supported by the subspace, $\mathcal{S}=\text{span}({\mathbf{z}}^{[M]\backslash(m)})$, spanned by ${\mathbf{z}}^{[M]\backslash(m)}$. Therefore, instead of forcing all modality embeddings to become identical, we align each modality embedding with its projection $\bar{\mathbf{z}}^{(m)}$ onto $\mathcal{S}$. As Figure \ref{fig:halfcol} shows, $\bar{\mathbf{z}}^{(m)}$, the projection of ${\mathbf{z}}^{(m)}$ onto $\mathcal{S}$, should be close to ${\mathbf{z}}^{(m)}$. For the $n$-th sample, we can use ${\mathbf{z}}_n^{[M]\backslash (m)}$ to form a matrix:
\begin{align}
    A_n = [{\mathbf{z}}_n^{(1)}\dots {\mathbf{z}}_n^{(m-1)}\ {\mathbf{z}}_n^{(m+1)}\dots {\mathbf{z}}_n^{(M)}],
\end{align}
where $A_n\in \mathbb{R}^{d\times (M-1)}$. In this setting, we can compute the projection $\bar{\mathbf{z}}^{(m)}$ in the closed-form expression:
\begin{align}
   \bar{\mathbf{z}}^{(m)} = A_n(A_n^\top A_n + \lambda I)^{-1}A_n^\top {\mathbf{z}}^{(m)} \in\mathbb{R}^{d},
\end{align}
where $\lambda$ is a small numerical constant to guarantee the inverse. This closed-form projection incurs $\mathcal{O}(dM^2)$ complexity per sample. Given typical settings where $d \gg M$, it is substantially more efficient than an MLP-based projector ($\mathcal{O}(d^2M)$) and introduces no learnable parameters. Using this method, we obtain $t$ to project the remaining-modality embeddings into $\mathbb{R}^d$, then we can compute the cosine similarity between ${\mathbf{z}}^{(m)}$ and $\bar{\mathbf{z}}^{(m)}$ for alignment. 

The overall sufficiency objective is defined as:
\begin{align}
    \mathcal{L}_{S}(\theta)= \frac{1}{M}\sum_{m=1}^M \mathcal{L}_{\text{InfoNCE}}^{(m)}(\theta).
    \label{eq:overall_sufficiency}
\end{align}
By Theorem \ref{thm:sandwich bound for dual total correlation},  minimizing \(\mathcal L_S\) is theoretically maximizing the dual total correlation (\text{DTC}) \cite{HAN1975337}.
\begin{theorem}
\label{thm:sandwich bound for dual total correlation}
    Optimizing Eq. (\ref{eq:m_sufficiency}) approximately maximizes the dual total correlation ($\mathrm{DTC}$) among $\mathbf{z}^{(1)},\dots, \mathbf{z}^{(M)}$, which is defined as 
    \begin{align}
        \mathrm{DTC}(\mathbf{z}^{[M]})= H(\mathbf{z}^{[M]}) -\sum_{m=1}^M H(\mathbf{z}^{(m)}|\mathbf{z}^{[M]\backslash(m)}).
    \end{align}
    The theoretical basis is the following sandwich bound, which shows that $\mathrm{DTC}$ is bounded from both sides by scaled forms of our One-vs-All sufficiency term in Eq.~(\ref{eq:m_sufficiency}):
    \begin{align}
         \frac{1}{M}\sum_{m=1}^M I(\mathbf{z}^{(m)}; \mathbf{z}^{[M]\backslash (m)})\leq \mathrm{DTC}(\mathbf{z}^{[M]}) \leq \frac{M-1}{M}\sum_{m=1}^M I(\mathbf{z}^{(m)}; \mathbf{z}^{[M]\backslash (m)}).
    \end{align}
\end{theorem}

\begin{remark}
Our sufficiency objective is closely related to recent beyond-pairwise alignment methods such as Symile, but differs in both its information-theoretic form and its interpretation. Symile maximizes total correlation \cite{HAN1975337} ($\mathrm{TC}$), defined as:
\begin{align}
    \mathrm{TC}(\mathbf z^{[M]})
= \sum_{m=1}^{M} H(\mathbf z^{(m)})-H(\mathbf z^{[M]}).
\end{align}
Our objective introduces two immediate advantages over Symile. First, for the sufficiency term, $\mathrm{TC}$ measures dependence among all modality representations as a whole, while $\mathrm{DTC} $ explicitly specifies what information each modality should preserve with respect to the remaining modalities, which directly matches the sufficiency definition. Figure~\ref{fig:venn} intuitively shows the distinction between $\mathrm{DTC}$ and $\mathrm{TC}$. Second, beyond sufficiency, our additional minimality objective provided in the following section compresses modality-specific nuisance information.
\end{remark}

\begin{figure}[!htbp]
    \centering \includegraphics[width=0.65\textwidth]{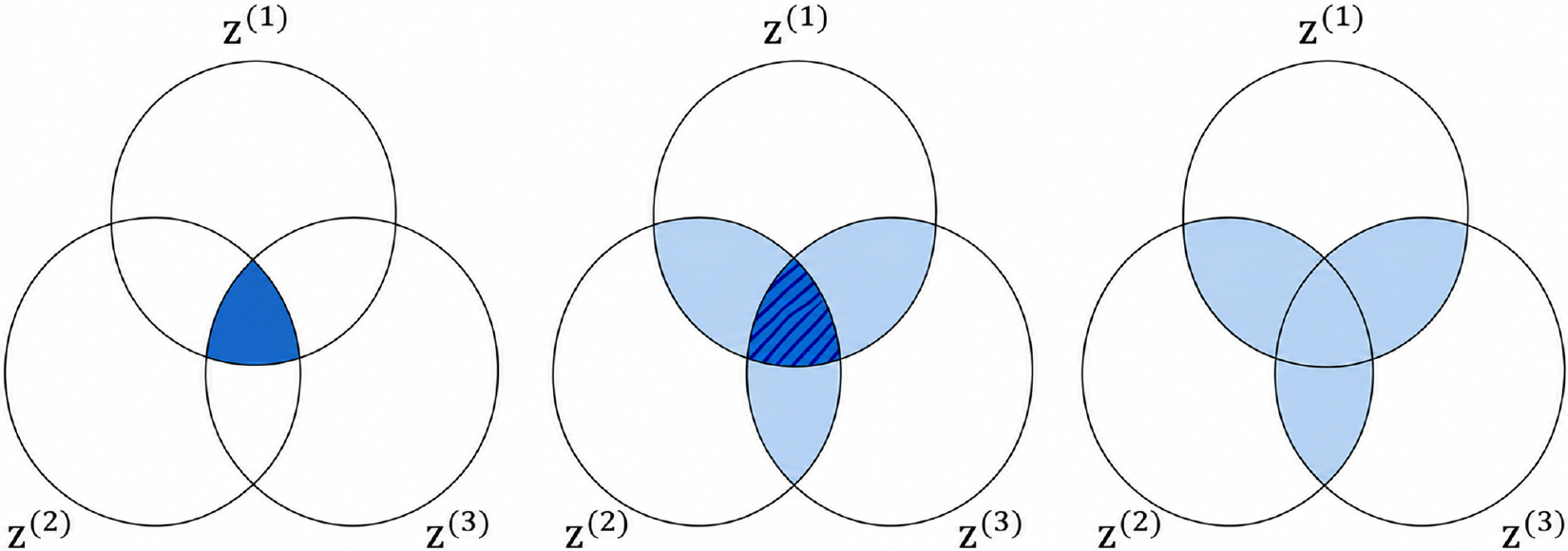}
    \caption{Intuitive illustration of multivariate dependence among three modality representations
\(\mathbf z^{(1)}, \mathbf z^{(2)}, \mathbf z^{(3)}\). Left: common agreement among all modalities. Middle: \text{TC}-style dependence compares the joint distribution with the product of marginals. Right: \text{DTC}-style dependence measures how each modality is predictable from the remaining modalities. The diagram is intended to provide intuition rather than a formal information decomposition.}
    \label{fig:venn}  
\end{figure}
\subsection{Multi-Modal Minimality}
For a representation $\mathbf{z}^{(m)}$ to be minimal, it must satisfy $I(\mathbf{z}^{(m)}; \mathbf{n}^{(m)}) = 0$, which is equivalent to minimizing $I(\mathbf{z}^{(m)}; \mathbf{n}^{(m)})$. Direct computation is intractable due to the unobserved $\mathbf{n}^{(m)}$. To obtain a tractable surrogate, we leverage the structural dependency $\mathbf{n}^{(m)} \leftarrow \mathbf{x}^{(m)} \rightarrow \mathbf{z}^{(m)}$. By Data Processing Inequality, this implies $I(\mathbf{z}^{(m)}; \mathbf{n}^{(m)}) \leq I(\mathbf{z}^{(m)}; \mathbf{x}^{(m)})$, providing a valid upper bound. Consequently, we minimize $I(\mathbf{z}^{(m)}; \mathbf{x}^{(m)})$ as a tractable upper-bound surrogate for the minimality. Integrating over the representation and input spaces to compute \(I(\mathbf z^{(m)};\mathbf x^{(m)})\) is generally intractable. However, all projection heads map modality-specific representations into the same shared representation space \(\mathcal Z\subseteq\mathbb R^d\). 
Therefore, for \(\mathbf x^{(m)}\), \(p_{\theta^{(m)}}(\mathbf{z}| \mathbf{x}^{(m)})\) denotes the conditional density of the projected embedding of \(\mathbf x^{(m)}\), evaluated at the same shared-space coordinate \(\mathbf{z}\in\mathcal Z\). 

\begin{theorem}
\label{prop:m_kl_minimality}
For each modality \(\mathbf x^{(m)}\), the minimality term admits the following upper bound:
\begin{align}
    I(\mathbf z^{(m)};\mathbf x^{(m)})
\leq
\mathbb E_{p(\mathbf x^{[M]})}
\left[
D_{\mathrm{KL}}
\left(
p_{\theta^{(m)}}(\mathbf z| \mathbf x^{(m)})
\;\middle\|\;
p_{\theta^{[M]\backslash(m)}}(\mathbf z | \mathbf x^{[M]\backslash(m)})
\right)
\right],
\label{eq:m_kl_minimality}
\end{align}
where \(
p_{\theta^{[M]\backslash(m)}}(\mathbf z | \mathbf x^{[M]\backslash(m)})
=
\frac{
\prod_{i\neq m}p_{{\theta^{(i)}}}(\mathbf z|\mathbf{x}^{(i)})
}{
\int \prod_{i\neq m}p_\theta^{(i)}(\mathbf u| \mathbf {x}^{(i)})\,d\mathbf u
}
\) is the normalized product of distributions induced by the remaining modalities, and $D_{\text{KL}}(\cdot|| \cdot )$ is Kullback–Leibler (KL) divergence \cite{Kullback1951OnIA}.
\end{theorem}
This result provides a tractable surrogate for minimizing \(I(\mathbf z^{(m)};\mathbf x^{(m)})\): by minimizing the upper bound in Eq.~(\ref{eq:m_kl_minimality}), the representation distribution of \(\mathbf x^{(m)}\) is encouraged to match the normalized product of the representation distributions induced by the remaining modalities. Thus, the minimality objective also follows a One-vs-All paradigm. The upper bound in Theorem \ref{prop:m_kl_minimality} does not have a closed form in general. Following recent theoretical insights that InfoNCE objectives induce Gaussian-like representations in high dimensions \cite{betser2026infonce}, we assume isotropic Gaussian distributions in the shared embedding space to derive a closed-form upper bound.
\begin{proposition}
\label{prop:gaussain_m_minimality}
    Assume that each modality-induced representation distribution in the shared representation space is an isotropic Gaussian, $p_{\theta^{(m)}}(\mathbf z|\mathbf x^{(m)})=\mathcal N\left(\mathbf z;\mu_{{\theta}^{(m)}}(\mathbf x^{(m)}),\sigma^2 I\right)$. For a modality \(\mathbf x^{(m)}\), the \(\mu_{\theta}^{(m)}(\mathbf x^{(m)})\)-dependent part of the upper-bound term in Theorem \ref{prop:m_kl_minimality} satisfies
    \begin{equation}
\label{eq:gaussain_m_minimality} 
\mathbb{E}_{p(\mathbf{x}^{[M]})}\!\bigl[D_{\mathrm{KL}}\!\bigl(p_{\theta^{(m)}}(\mathbf{z}|\mathbf{x}^{(m)}) \,\big\|\, p_{\theta^{[M]\backslash (m)}}(\mathbf{z}\mid\mathbf{x}^{[M]\backslash (m)})\bigr)\bigr]
\propto \mathbb{E}_{p(\mathbf{x}^{[M]})}\!\bigl[\bigl\|\mu_{\theta^{(m)}}(\mathbf{x}^{(m)}) - \bar{\mu}\bigr\|_2^2\bigr],
\end{equation}
where $\bar{\mu}=\frac{1}{M-1}\sum_{i=1, i\ne m}^M \mu_{\theta^{(i)}}(\mathbf{x}^{(i)})$.
\end{proposition} 
Proposition \ref{prop:gaussain_m_minimality} provides the tractable minimality objective for $\mathbf{x}^{(m)}$:
\begin{align}
\label{eq: final_minimality}
    \mathcal{L}^{(m)}_{M}(\theta)= \mathbb{E}_{p(\mathbf x^{[M]})}\left[||\mu_{\theta^{(m)}}(\mathbf{x}^{(m)})-\bar{\mu}||_2^2\right],
\end{align}
where the expectation is empirically approximated by the sample average over the sample size $N$. For the overall minimality objective across all modalities, we average the modality-wise losses:
\begin{align}
    \mathcal{L}_{M}(\theta) = \frac{1}{M}\sum_{m=1}^M \mathcal{L}_{M}^{(m)},
    \label{eq:overall_minimality}
\end{align}
where minimizing $\mathcal{L}_M$ provides a tractable surrogate for controlling $\frac{1}{M}\sum_{m=1}^M I(\mathbf{z}^{(m)}; \mathbf{x}^{(m)})$.

\subsection{Unified Multi-Modal Information Bottleneck}

Combining the sufficiency and minimality principles gives the following multi-modal IB: 
\begin{align}
    \max_{\mathbf{z}^{[M]}} \frac{1}{M}\sum_{m=1}^M \left(I(\mathbf{z}^{(m)};\mathbf{x}^{[M]\backslash(m)}) - \beta I(\mathbf{z}^{(m)}; \mathbf{x}^{(m)})\right),
\end{align}
 where the first term encourages the representation of each modality to preserve information related to the remaining modalities, while the second term controls excessive dependence on modality-specific nuisance information. Using the tractable objectives above, we obtain the final training objective:
\begin{align}
    \min_{\theta}
\mathcal L(\theta)
=
\mathcal L_S(\theta)
+
\beta \mathcal L_M(\theta).
\label{eq:final_objective}
\end{align}

\section{Experiments}
\label{sec:experiments}

We comprehensively evaluate OVA-IB across all-modality downstream tasks, modality-agnostic robustness using different modality combinations on downstream tasks, cross-modal retrieval, and component-wise ablations of the minimality regularizer and projection strategy.

\subsection{Experimental Setup and Benchmarks}

We evaluate OVA-IB on four multi-modal benchmarks: Vision \& Touch, MuJoCo Push, CMU-MOSEI, and WESAD. These datasets cover robotic perception, physical interaction, audio-visual-language understanding, and wearable physiological sensing. We further construct four- and five-modality variants from Vision \& Touch to evaluate scalability to higher-modality settings. For Vision \& Touch, MuJoCo Push, and CMU-MOSEI, we follow the standardized MultiBench protocol with fixed input representations and evaluation splits across methods. Since WESAD is not part of MultiBench, we apply identical synchronized-window preprocessing and subject-wise splits to all methods for fair comparison. The dataset details are provided in Appendix \ref{sec:dataset_detail}.

For classification and regression, we first pre-train modality encoders using each alignment objective. We then concatenate the learned modality representations and train a 4-layer MLP predictor with cross-entropy for classification and mean squared error for regression. We set $\beta=1$ for all benchmarks except WESAD ($\beta=0.5$).  The downstream tasks include trajectory-pair classification and orientation regression on Vision–Force–Depth (VFD), contact classification on Vision–Force–Proprioception (VFP), higher-modality trajectory classification on VFPD and VFPDO, object-position regression on MuJoCo Push, binary sentiment classification on CMU-MOSEI, and three-class affective-state classification on WESAD. For classification and regression, we evaluate OVA-IB across multiple datasets and task types, providing broad evidence of stability on supervised downstream tasks. For cross-modal retrieval, we select a task with clear practical relevance rather than arbitrary modality matching. WESAD provides such a setting, where wrist-side BVP is physiologically related to chest-side respiration and ECG through cardiopulmonary and pulse dynamics. Since suitable retrieval settings with clear practical relevance are limited, we repeat the $\text{Resp}+\text{ECG}\to \text{BVP}$ retrieval experiment three times with different random seeds and report mean \(\pm\) standard deviation of mAP (mean average precision). Implementation details are provided in the Appendix \ref{sec:experimental_implementation}.

\subsection{All-Modality Downstream Evaluation}

We first evaluate all methods under the all-modality setting, where all modalities are used during downstream evaluation. This setting measures whether the learned representations provide useful task-relevant information from all multi-modal observations. Table~\ref{tab:full_modality_downstream} reports the classification and regression results. OVA-IB achieves the best performance on most benchmarks, with particularly large gains as the number of modalities increases. On VFPDO, OVA-IB outperforms Symile and Pairwise CLIP by 12.95\% and 2.96\%, respectively. It also obtains the lowest MSE on both regression tasks. On MOSEI, where standardized pre-extracted features leave limited representation-level headroom, OVA-IB remains competitive but is not the top method. Overall, all-modality results show that OVA-IB learns compact and informative representations across both classification and regression.

\begin{table}[!htbp]
  \centering
  \renewcommand{\arraystretch}{0.9}
  \setlength{\tabcolsep}{3.5pt}
 \small
  \caption{All-modality downstream performance on classification and regression tasks. Accuracy (\%) is reported for classification, and MSE is reported for regression.}
  \label{tab:full_modality_downstream}
  \begin{tabular}{lcccccccc}
    \toprule
    \multirow{2}{*}{Method} 
    & \multicolumn{6}{c}{Classification (Acc \%)} 
    & \multicolumn{2}{c}{Regression (MSE)} \\
    \cmidrule(lr){2-7} \cmidrule(lr){8-9}
    & VFP & VFD & VFPD & VFPDO & MOSEI & WESAD & VFD & MuJoCo \\
    \midrule
    Ours               
    & \textbf{97.12} & \textbf{95.08} & \textbf{94.73} & \textbf{95.91} 
    & 76.02 & \textbf{76.18} 
    & \textbf{0.0003} & \textbf{0.0340} \\
    
    Symile             
    & 96.44 & 92.57 & 84.97 & 82.96 
    & \textbf{77.00} & 73.42 
    & 0.0015 & 0.1228 \\
    
    Gram               
    & 96.85 & 89.91 & 90.81 & 90.59 
    & 72.77 & 66.22 
    & 0.0018 & 0.0425 \\
    
    TRIANGLE           
    & 95.91 & 93.27 & -- & -- 
    & 76.13 & 56.87 
    & 0.0019 & 0.0656 \\
    
    Pairwise CLIP      
    & 96.55 & 93.77 & 93.25 & 92.95 
    & 75.05 & 70.83 
    & 0.0013 & 0.0810 \\
    
    Pairwise CLIP + IB 
    & 97.10 & 94.70 & 94.45 & 94.68 
    & 76.25 & 67.91 
    & 0.0009 & 0.3385 \\
    \bottomrule
  \end{tabular}
\end{table}

\subsection{Modality-Agnostic Evaluation}
\label{sec:modality_agnostic_evaluation}
We next evaluate modality-agnostic robustness, where models are pre-trained with all modalities but tested under different modality combinations. This setting measures whether the learned representation space remains effective when the availability of modality changes at test time. Table~\ref{tab:vfd_classification} reports a representative VFD trajectory-classification result, while Appendix \ref{sec:appendix-modality-agnostic experiment} provides the full modality-combination results on all classification and regression tasks. OVA-IB achieves the highest or near-highest performance across most modality combinations. Pairwise CLIP + IB occasionally leads on single- or dual-modality inputs, which is expected because pairwise objectives directly optimize specific modality pairs. In contrast, OVA-IB shows stronger gains when multiple modalities are available, indicating that aligning each modality with the remaining modalities yields representations that are less tied to a fixed pairwise configuration. 

The same trend holds across the full results in Appendix \ref{sec:appendix-modality-agnostic experiment}. On regression tasks such as MuJoCo Push and VFD orientation prediction, OVA-IB maintains the lowest MSE across modality combinations. On classification tasks, its advantage becomes more pronounced as the number of modalities increases. MOSEI is comparatively saturated due to its standardized pre-extracted features, but OVA-IB remains competitive on multi-modal combinations. Overall, these results show that OVA-IB also improves robustness under varying modality combinations.

\begin{table}[!htbp]
\vspace{-5pt}
  \centering
  \small
  \setlength{\tabcolsep}{3pt}
  \caption{Modality-combination evaluation on VFD trajectory classification.}
  \label{tab:vfd_classification}
  % 使用 l*{7}{c} 明确指定 1左+7居中，彻底杜绝列数不匹配报错
  \begin{tabular}{l*{7}{c}}
    \toprule
    Method & V & F & D & V+F & V+D & F+D & V+F+D \\
    \midrule
    Ours               & 94.06 & \textbf{93.31} & 91.13 & \textbf{95.55} & 94.20 & \textbf{94.14} & \textbf{95.08} \\
    Symile             & 87.65 & 92.68 & 87.23 & 89.76 & 87.44 & 88.62 & 92.57 \\
    Gram               & 92.16 & 92.25 & 86.27 & 93.12 & 92.35 & 88.16 & 89.91 \\
    TRIANGLE           & 92.40 & 92.21 & 86.95 & 92.66 & 92.75 & 88.75 & 93.27 \\
    Pairwise CLIP      & 93.05 & 92.27 & 90.91 & 93.50 & 93.44 & 91.46 & 93.77 \\
    Pairwise CLIP + IB & \textbf{94.36} & 92.37 & \textbf{91.36} & 95.08 & \textbf{94.29} & 93.09 & 94.70 \\
    \bottomrule
  \end{tabular}
\end{table}

\subsection{Cross-Modal Retrieval}
\begin{wraptable}[12]{r}{0.5\textwidth}
   %\vspace{-\dimexpr\baselineskip+\abovecaptionskip} % 核心：上移一行高度，使顶部与左侧文本顶部严格对齐
  \vspace{-\abovecaptionskip}
  \centering
  \small
  %\caption{Cross-modal retrieval performance on WESAD. The task is Resp+ECG \(\to\) BVP retrieval with mAP as the metric. Results are averaged over 3 runs with standard deviation.}
  \caption{Cross-modal retrieval performance on WESAD. The task is Resp+ECG \(\to\) BVP retrieval with mAP as the metric.}
  \label{tab:wesad_retrieval}
  \begin{tabular}{lc}
    \toprule
    Method & Resp+ECG \(\to\) BVP \\
    \midrule
    Ours               & \textbf{0.3004 \(\pm\) 0.0067}\\
    Symile             & 0.2977 \(\pm\) 0.0125 \\
    Gram               & 0.2974 \(\pm\) 0.0019 \\
    TRIANGLE           & 0.2846 \(\pm\) 0.0134 \\
    Pairwise CLIP      & 0.2918 \(\pm\) 0.0057 \\
    Pairwise CLIP + IB & 0.2938 \(\pm\) 0.0187 \\
    \bottomrule
  \end{tabular}
\end{wraptable}
We further evaluate OVA-IB on zero-shot cross-modal retrieval using WESAD. Unlike downstream classification and regression, this task directly tests whether the learned representation space aligns physiological signals across devices. Specifically, chest respiration and ECG are used as the query, and the goal is to retrieve the synchronized wrist BVP signal from a candidate pool. This setting is physiologically grounded: ECG reflects cardiac electrical activity, respiration provides cardiopulmonary context, and BVP captures peripheral pulse dynamics at the wrist.

Each method is evaluated using the scorer induced by its own pre-training objective: MIP for Symile, Gramian volume for Gram, triangle area for TRIANGLE, averaged cosine similarity for Pairwise CLIP, and the projection score for OVA-IB. This avoids mismatched scoring functions and ensures that each method is evaluated under its intended similarity metric. OVA-IB obtains the highest mAP on $\text{Resp}+\text{ECG} \to\text{BVP}$ retrieval. Although the margin is modest, the task is cross-device and zero-shot, making the improvement complementary to the supervised downstream evaluation. This indicates that the learned representation space is useful for direct cross-modal matching.

\subsection{Ablation Study on $\beta$ and Projection Strategy}
Finally, we ablate the two main components of OVA-IB: the One-vs-All minimality regularizer and the projection strategy. Table \ref{tab:joint_ablation} reports the effect of removing the minimality term by setting $\beta=0$, as well as replacing the closed-form geometric projection with a learnable MLP projector. Removing the minimality term leads to consistent degradation across 7 of 8 benchmarks. The drop becomes larger as the number of modalities increases, with the largest decline on the five-modality VFPDO task. In regression, removing the minimality term substantially increases prediction error on both VFD and MuJoCo Push. These results confirm that the One-vs-All minimality term helps suppress modality-specific nuisance information and improves the transferability of the learned representations.

The MLP projector remains competitive on several classification tasks, suggesting that the One-vs-All objective itself provides a strong alignment signal. However, it becomes less stable on regression benchmarks, especially MuJoCo Push, where the error increases substantially. This indicates that our geometry-aware projection provides a more efficient alignment mechanism than the MLP. CMU-MOSEI shows minimal sensitivity to either component, because pre-extracted standardized features consistently limit the available representational headroom. Conversely, on the less standardized WESAD dataset, the minimality term has a more pronounced effect on the downstream performance. Overall, the ablation results validate that the minimality term and the geometry-aware projection together provide a scalable and efficient alignment objective.

\begin{table}[!htbp]
  \centering
  \setlength{\aboverulesep}{0.2ex}
  \setlength{\belowrulesep}{0.2ex}
  \small 
  \caption{Ablation study on  $\beta$ and projection strategy. Accuracy (\%) and MSE are reported. $\Delta_{\beta} = \beta(\cmark) - \beta(\xmark)$, $\Delta_{\text{proj}} = \text{Ours} - \text{MLP}$. For MSE, negative $\Delta$ indicates improvement.}
  \label{tab:joint_ablation}
  \begin{tabular}{lccccccc}
    \toprule
    \multirow{2}{*}{Task} & \multirow{2}{*}{Dataset} & \multicolumn{2}{c}{$\beta$} & \multicolumn{2}{c}{Projector} & \multicolumn{2}{c}{Improvement} \\
    \cmidrule(lr){3-4} \cmidrule(lr){5-6} \cmidrule(lr){7-8}
                          &                          & \xmark & \cmark & MLP &  Ours & $\Delta_{\beta}$ & $\Delta_{\text{proj}}$ \\
    \midrule
    \multirow{6}{*}{Classification (Acc)} & VFP    & 96.22  & \textbf{97.12} & 96.93 & \textbf{97.12} & +0.90 & +0.19 \\
                                          & VFD    & 91.30  & \textbf{95.08} & 95.08 & \textbf{95.08} & +3.78 & +0.00 \\
                                          & VFPD   & 91.98  & \textbf{94.73} & 94.50 & \textbf{94.73} & +2.75 & +0.23 \\
                                          & VFPDO  & 89.49  & \textbf{95.91} & 94.24 & \textbf{95.91} & +6.42 & +1.67 \\
                                          & MOSEI  & \textbf{76.10} & 76.02 & \textbf{76.36} & 76.02 & -0.08 & -0.32 \\
                                          & WESAD & 65.09 & \textbf{76.18} & 70.66 &  \textbf{76.18} & +11.09 & +5.52\\
    \midrule
    \multirow{2}{*}{Regression (MSE)}     & VFD    & 0.0022 & \textbf{0.0003} & 0.0013 & \textbf{0.0003} & -0.0019 & -0.0010 \\
                                          & MuJoCo & 0.0817 & \textbf{0.0340} & 0.3750 & \textbf{0.0340} & -0.0477 & -0.3410 \\
    \bottomrule
  \end{tabular}
\end{table}

\section{Limitations}
Similar to competing approaches such as Symile and Gram, our evaluation pre-trains modality-specific encoders from scratch, which constrains experiments to moderately-sized datasets. Extending OVA-IB to leverage pre-trained foundation models is a promising direction for future work. Additionally, OVA-IB assumes complete modality availability during pre-training, as each modality is aligned with the complementary context provided by the remaining modalities. Although we demonstrate modality-agnostic robustness at test time across arbitrary modality combinations, adapting the framework to handle missing modalities during pre-training remains an important future direction.

\section{Conclusion}
\label{sec:conclusion}
We introduced OVA-IB, a One-vs-All Information Bottleneck framework for arbitrary-modality alignment. By defining sufficiency and minimality with respect to the remaining modalities, OVA-IB derives a DTC-style contrastive sufficiency objective and a tractable minimality regularizer for suppressing modality-specific nuisance information. Together with a geometry-aware projection, the resulting objective provides an efficient objective for multi-modal alignment. Experiments on classification, regression, modality-agnostic evaluation, and cross-modal retrieval show that OVA-IB learns compact and robust representations, especially as the number of modalities increases.

\newpage
\bibliographystyle{plain}
\bibliography{custom} % 对应 references.bib
\newpage

\appendix
% ---------- Appendix spacing fix ----------
\allowdisplaybreaks[4]
\raggedbottom

\setlength{\abovedisplayskip}{4pt plus 1pt minus 1pt}
\setlength{\belowdisplayskip}{4pt plus 1pt minus 1pt}
\setlength{\abovedisplayshortskip}{3pt plus 1pt minus 1pt}
\setlength{\belowdisplayshortskip}{3pt plus 1pt minus 1pt}
\setlength{\jot}{2pt}

\setlength{\textfloatsep}{6pt plus 1pt minus 2pt}
\setlength{\floatsep}{6pt plus 1pt minus 2pt}
\setlength{\intextsep}{6pt plus 1pt minus 2pt}

\section{Proofs of Section \ref{sec:methodology}}
\label{sec:proof of section}
\subsection{Proof of Theorem \ref{thm:m_sufficency_objective}}
\label{sec:m_sufficiency_objective}

\begin{proof}
Under Definition \ref{ass:markov chain}, we have the Markov chains:
\begin{align}
    \mathbf{y}^{(m)} \leftrightarrow \mathbf{x}^{[M]\backslash(m)} \leftrightarrow\mathbf{x}^{(m)},\\
    \mathbf{x}^{[M]\backslash(m)} \leftrightarrow \mathbf{y}^{(m)} \leftrightarrow \mathbf{x}^{(m)}.
\end{align}

Since \(\mathbf z^{(m)}\) is encoded from \(\mathbf x^{(m)}\) via $f^{(m)}$ and $g^{(m)}$ and it does not directly access \(\mathbf y^{(m)}\) or \(\mathbf x^{[M]\backslash(m)}\), combining this encoder relation with the Markov chains in Definition \ref{ass:markov chain} yields:
\begin{align}
    \mathbf{y}^{(m)}\rightarrow \mathbf{x}^{[M]\backslash(m)} \rightarrow \mathbf{x}^{(m)}\rightarrow\mathbf{z}^{(m)}. 
\end{align}

%Therefore, together with the Markov Chains in Definition~\ref{ass:markov chain}, we obtain the data-processing chains:
% \begin{align}
%     \mathbf{z}^{(m)} \leftrightarrow \mathbf{x}^{(m)} \leftrightarrow\mathbf{y}^{(m)}.
% \end{align}

 By the Data Processing Inequality (DPI) \cite{Cover2005ElementsOI}:
\begin{align}
    I(\mathbf{z}^{(m)};\mathbf{y}^{(m)}) \leq I(\mathbf{z}^{(m)}; \mathbf{x}^{[M]\backslash(m)}),
\end{align}
where $I(\cdot;\cdot)$ denotes the mutual information \cite{Cover2005ElementsOI}. Since $ \mathbf{x}^{[M]\backslash(m)} \leftrightarrow \mathbf{y}^{(m)} \leftrightarrow \mathbf{x}^{(m)}$, and $\mathbf{z}^{(m)}$ is stochastically determined by $\mathbf{x}^{(m)}$, we can obtain another Markov chain:
% \begin{align}
%     \mathbf{x}^{[M]\backslash(m)} \leftrightarrow \mathbf{y}^{(m)} \leftrightarrow \mathbf{z}^{(m)}
% \end{align}
\begin{align}
    \mathbf{x}^{[M]\backslash(m)} \rightarrow \mathbf{y}^{(m)} \rightarrow \mathbf{x}^{(m)} \rightarrow \mathbf{z}^{(m)}.
\end{align}
Applying DPI to this chain yields:
\begin{align}
    I(\mathbf{z}^{(m)};\mathbf{y}^{(m)}) \geq I(\mathbf{z}^{(m)};\mathbf{x}^{[M]\backslash(m)}).
\end{align}
Hence, we can conclude that 
\begin{align}
I(\mathbf{z}^{(m)};\mathbf{y}^{(m)})=I(\mathbf{z}^{(m)};\mathbf{x}^{[M]\backslash(m)}).
\end{align}

\end{proof}

\subsection{Donsker–Varadhan representation for multivariate mutual information}

\begin{theorem}[Donsker–Varadhan representation for multivariate mutual information]
\label{thm:dv representation multivariate}
Let $  P(\mathbf{x},\mathbf{z}_1,\dots,\mathbf{z}_k)  $ be the joint distribution of the random variables $  (\mathbf{x}, \mathbf{z}_1, \dots, \mathbf{z}_k)  $. Define the product of the marginals $Q=P(\mathbf{x}) P( \mathbf{z}_1, \dots, \mathbf{z}_k)$, where $  P(\mathbf{z}_1,\dots,\mathbf{z}_k)   $ is the joint marginal of the $  \mathbf{z}_i  $'s. The multivariate mutual information is exactly the KL divergence ($I(\mathbf{x}; \mathbf{z}_1,\dots,\mathbf{z}_k) := D_{\mathrm{KL}}(P(\mathbf{x},\mathbf{z}_1,\dots,\mathbf{z}_k)\Vert Q)$).
Applying the classical Donsker–Varadhan representation directly to this KL divergence yields the following dual variational form:
\begin{align}
    I(\mathbf{x}; \mathbf{z}_1,\dots,\mathbf{z}_k) = \sup_{T} \ \mathbb{E}_{P_{X,Z_1,\dots,Z_k}}[T] - \log \mathbb{E}_{Q}[e^T],
\end{align}

where the supremum is taken over all (measurable) functions $  T: \mathcal{X} \times \mathcal{Z}_1 \times \cdots \times \mathcal{Z}_k \to \mathbb{R}  $ such that the two expectations exist and are finite.
\end{theorem}

\label{sec:proof of multivariate dv}
\begin{proof}
Define $P:=P(\mathbf{x}, \mathbf{z}_1,\dots, \mathbf{z}_k)$ and $Q:=P(\mathbf{x})P(\mathbf{z}_1,\dots,\mathbf{z}_k)$. Then, by the definition of mutual information between $\mathbf{x}$ and the tuple $(\mathbf{z}_1,\dots, \mathbf{z}_k)$,
\begin{align}
    I(\mathbf{x}; \mathbf{z}_1, \dots , \mathbf{z}_k) = D_{KL}(P||Q).
\end{align}
So it is enough to prove that 
\begin{align}
    D_{KL}(P||Q) = \sup_T \{\mathbb{E}_P[T] -\log \mathbb{E}_Q[e^T]\},
\end{align}
where $T:\mathcal{X} \times \mathcal{Z}_1 \times \dots \times\mathcal{Z}_k \rightarrow \mathbb{R}$ is any measurable function, such that both $\mathbb{E}_P[T]$ and $\mathbb{E}_Q[e^T]$ are finite. Define $Z:=\mathbb{E}_Q[e^T]$. Since $e^T>0$, we have $Z>0$. Now, we define a Gibbs distribution $G$:
\begin{align}
    dG=\frac{1}{Z}e^TdQ,
\end{align}
where $\int dG = \int \frac{1}{Z}e^TdQ=\frac{1}{Z}\mathbb{E}_Q[e^T]=1$. Equivalently,
\begin{align}
    \frac{dG}{dQ}=\frac{e^T}{Z}.
\end{align}
Taking logarithms on both sides:
\begin{align}
    \log \frac{dG}{dQ} = \log \frac{e^T}{Z} = T-\log Z.
\end{align}
Now take expectation under $P$ on both sides:
\begin{align}
    \mathbb{E}_P\left[\log \frac{dG}{dQ}\right] = \mathbb{E}_P\left[T-\log Z \right] .
\end{align}
Because $\log Z$ is a constant with respect to the random variable,
\begin{align}
    \mathbb{E}_P\left[T-\log Z \right] = \mathbb{E}_P\left[T\right] -\log Z . 
\end{align}
Therefore, 
\begin{align}
    \mathbb{E}_P\left[\log \frac{dG}{dQ}\right] = \mathbb{E}_P[T]-\log Z = \mathbb{E}_P[T]- \log \mathbb{E}_Q[e^T].
\end{align}
Define the gap $\Delta:=D_{KL}(P||Q)-\left(\mathbb{E}_P[T]- \log \mathbb{E}_Q[e^T]\right)$. Using the definition of KL divergence,
\begin{align}
    \Delta =\mathbb{E}_P\left[\log\frac{dP} {dQ}\right]-\left(\mathbb{E}_P[T]- \log \mathbb{E}_Q[e^T]\right) =\mathbb{E}_P\left[\log\frac{dP} {dQ}\right]- \mathbb{E}_P\left[\log \frac{dG}{dQ}\right].
\end{align}
Since expectation is linear, 
\begin{align}
    \Delta = \mathbb{E}_P\left[\log\frac{dP} {dQ}- \log \frac{dG}{dQ}\right] =\mathbb{E}_P\left[\log \frac{\frac{dP} {dQ}}{\frac{dG}{dQ}}\right] =\mathbb{E}_P\left[\log\frac{dP} {dG}\right]  .
\end{align}
By the definition of KL divergence,
\begin{align}
  \Delta = \mathbb{E}_P\left[\log\frac{dP} {dG}\right] = D_{KL}(P||G)\geq 0 .
\end{align}
Since this holds for every admissible $T$, we can take the supremum over all such $T$:
\begin{align}
    D_{KL}(P||Q)\geq \sup_T\left\{\mathbb{E}_P[T] - \log \mathbb{E}_Q\left[e^T\right]\right\}.
\end{align}
To show that this lower bound is tight, choose
\begin{align}
    T^* = \log \frac{dP}{dQ}+C,
\end{align}
where $C\in \mathbb{R}$ is any constant. Then
\begin{align}
    e^{T^*} = e^{\log \frac{dP}{dQ}+C}= e^C\frac{dP}{dQ}. 
\end{align}
So
\begin{align}
    \mathbb{E}_Q[e^{T^*}]&=\int e^{T^*} dQ=\int e^{C}\frac{dP}{dQ} dQ = e^C\int dP =e^C, \\
    \mathbb{E}_P[T^*]&=\mathbb{E}_P\left[\log \frac{dP}{dQ}+C\right] = D_{KL}(P||Q)+C.
\end{align}
Therefore, 
\begin{align}
   \mathbb{E}_P[T^*] -\log  \mathbb{E}_Q[e^{T^*}] = (D_{KL}(P||Q)+C)- \log e^C = D_{KL}(P||Q).
\end{align}
The equality holds, and Theorem \ref{thm:dv representation multivariate} is proved.
\end{proof}

\subsection{Proof of Theorem \ref{thm:infonce_sufficiency}}
\label{sec:infonce_sufficiency}
\begin{proof}
For simplicity, we denote $\mathbf{r}^{(m)}= f^{(m)}(\mathbf{x}^{(m)})$. First, we can rewrite the InfoNCE loss of the $m$-th modality as:

\begingroup
\small
\setlength{\abovedisplayskip}{4pt}
\setlength{\belowdisplayskip}{4pt}
\setlength{\jot}{2pt}

\begin{align}
\mathcal{L}^{(m)}_{\mathrm{InfoNCE}} 
&= -\frac{1}{N} \sum_{n=1}^N 
\log 
\frac{
\exp\!\left(
s({\mathbf{z}}_n^{(m)}, 
t({\mathbf{z}}_n^{[M]\backslash (m)}))/\tau
\right)}
{\sum_{n'=1,\, n' \neq n}^N 
\exp\!\left(
s({\mathbf{z}}_n^{(m)}, 
t({\mathbf{z}}_{n'}^{[M]\backslash (m)}))/\tau
\right)}
\notag \\[2pt]
&= -\frac{1}{N} \sum_{n=1}^N 
\Biggl[
\frac{
s({\mathbf{z}}_n^{(m)}, 
t({\mathbf{z}}_n^{[M]\backslash (m)}))
}{\tau}
\notag \\
&\hspace{2.6cm}
-
\log \sum_{n'=1,\, n' \neq n}^N 
\exp\!\left(
\frac{
s({\mathbf{z}}_n^{(m)}, 
t({\mathbf{z}}_{n'}^{[M]\backslash (m)}))
}{\tau}
\right)
\Biggr]
\notag \\[2pt]
&= -\frac{1}{N} \sum_{n=1}^N 
\frac{
s({\mathbf{z}}_n^{(m)}, 
t({\mathbf{z}}_n^{[M]\backslash (m)}))
}{\tau}
\notag \\
&\quad
+ \frac{1}{N} \sum_{n=1}^N 
\log \sum_{n'=1,\, n' \neq n}^N 
\exp\!\left(
\frac{
s({\mathbf{z}}_n^{(m)}, 
t({\mathbf{z}}_{n'}^{[M]\backslash (m)}))
}{\tau}
\right)
\notag \\[2pt]
&= -\frac{1}{N} \sum_{n=1}^N 
\frac{1}{\tau}
s\!\left(
g^{(m)}(\mathbf{r}^{(m)}),\,
t\!\left(g^{[M]\backslash (m)}
(\mathbf{r}_n^{[M]\backslash (m)})\right)
\right)
\notag \\
&\quad
+ \frac{1}{N} \sum_{n=1}^N 
\log \sum_{n'=1,\, n' \neq n}^N 
\exp\!\Biggl(
\frac{1}{\tau}
s\!\left(
g^{(m)}(\mathbf{r}_n^{(m)}),\,
t\!\left(g^{[M]\backslash (m)}
(\mathbf{r}_{n'}^{[M]\backslash (m)})\right)
\right)
\Biggr),
\end{align}
\endgroup
where $g^{[M]\backslash (m)}(\mathbf{r}_n^{[M]\backslash (m)})=\underbrace{g^{(1)}(\mathbf{r}_{n}^{(1)}), \dots, g^{(M)}(\mathbf{r}_{n}^{(M)})}_{\text{excluding } m}$.
We can rewrite the above as the expectation form and therefore remove the subscript $n$:

\begingroup
\small
\setlength{\abovedisplayskip}{4pt}
\setlength{\belowdisplayskip}{4pt}
\setlength{\abovedisplayshortskip}{3pt}
\setlength{\belowdisplayshortskip}{3pt}
\setlength{\jot}{2pt}

\begin{align}
\mathcal{L}^{(m)}_{\mathrm{InfoNCE}}
&= -\mathbb{E}_{P(\mathbf{r}^{[M]})} \Biggl[
\frac{1}{\tau}
s\bigl(
g^{(m)}(\mathbf{r}_n^{(m)}),
g^{[M]\backslash (m)}(\mathbf{r}_n^{[M]\backslash (m)})
\bigr)
\Biggr]
\notag\\
&\quad + \mathbb{E}_{P(\mathbf{r}^{(m)})} \Biggl[
\log \mathbb{E}_{P(\mathbf{r}^{[M]\backslash (m)})} \Biggl[
\exp\Biggl(
\frac{1}{\tau}
s\bigl(
g^{(m)}(\mathbf{r}_n^{(m)}),
\notag\\[-1pt]
&\hspace{5.3cm}
g^{[M]\backslash (m)}(\mathbf{r}_n^{[M]\backslash (m)})
\bigr)
\Biggr)
\Biggr]
\Biggr] - \log N
\\
&= \mathbb{E}_{P(\mathbf{r}^{(m)})} \Biggl[
-\mathbb{E}_{P(\mathbf{r}^{[M]\backslash (m)}| \mathbf{r}^{(m)})} \Biggl[
\frac{1}{\tau}
s\bigl(
g^{(m)}(\mathbf{r}_n^{(m)}),
\notag\\[-1pt]
&\hspace{5.3cm}
g^{[M]\backslash (m)}(\mathbf{r}_n^{[M]\backslash (m)})
\bigr)
\Biggr]
\notag\\[-1pt]
&\quad + \log \mathbb{E}_{P(\mathbf{r}^{[M]\backslash (m)})} \Biggl[
\exp\Biggl(
\frac{1}{\tau}
s\bigl(
g^{(m)}(\mathbf{r}_n^{(m)}),
\notag\\[-1pt]
&\hspace{5.3cm}
g^{[M]\backslash (m)}(\mathbf{r}_n^{[M]\backslash (m)})
\bigr)
\Biggr)
\Biggr]
\Biggr] - \log N
\\
&= \mathbb{E}_{P(\mathbf{r}^{(m)})} \Biggl[
-\mathbb{E}_{P(\mathbf{r}^{[M]\backslash (m)}|\mathbf{r}^{(m)})}
\Bigl[ T(\mathbf{r}^{[M]}) \Bigr]
\notag\\[-1pt]
&\quad + \log \mathbb{E}_{P(\mathbf{r}^{[M]\backslash (m)})}
\Bigl[ e^{T(\mathbf{r}^{[M]})} \Bigr]
\Biggr] - \log N
\\
&= -\Biggl(
\mathbb{E}_{P(\mathbf{r}^{(m)})} \Biggl[
\mathbb{E}_{P(\mathbf{r}^{[M]\backslash (m)}| \mathbf{r}^{(m)})}
\Bigl[ T(\mathbf{r}^{[M]}) \Bigr]
\notag\\[-1pt]
&\quad - \log \mathbb{E}_{P(\mathbf{r}^{[M]\backslash (m)})}
\Bigl[ e^{T(\mathbf{r}^{[M]})} \Bigr]
\Biggr]
\Biggr) - \log N
\\
&\geq -I(\mathbf{r}^{(m)};\mathbf{r}^{[M]\backslash (m)}) - \log N .
\end{align}

\endgroup
where $\mathbf{r}^{[M]}= \mathbf{r}^{(1)},\dots, \mathbf{r}^{(M)}$, $P(\mathbf{r}^{[M]})$ and  $P({\mathbf{r}^{[M]\backslash (m)}})$ are the joint distributions, and $P({\mathbf{r}^{[M]\backslash (m)}|\mathbf{r}^{(m)}})$, $P(\mathbf{r}^{(m)})$ are respectively conditional and marginal distributions, and $T$ is a mapping that we parameterize with  temperature factor $\tau$, the over-parameterized modality-specific projection head $g^{(m)}$ and the score function $s$. Based on Theorem \ref{thm:dv representation multivariate}, minimizing $\mathcal{L}_{\mathrm{InfoNCE}}$ is equivalent to maximizing a lower bound of $I(\mathbf{r}^{(m)}; \mathbf{r}^{[M]\backslash (m)})$. Since $\mathbf{z}^{(m)}$ is stochastically determined by $\mathbf{r}^{(m)}$, it also holds that minimizing $\mathcal{L}^{(m)}_{\mathrm{InfoNCE}}$ maximizes the lower bound of $I(\mathbf{z}^{(m)}; \mathbf{z}^{[M]\backslash (m)})$. 

\end{proof}

\subsection{Proof of Theorem~\ref{thm:sandwich bound for dual total correlation}}

\begin{proof}
For brevity, denote
\begin{align}
    \mathrm{TC}(\mathbf z^{[M]})
=
\sum_{m=1}^{M}H(\mathbf z^{(m)})
-
H(\mathbf z^{[M]}),
\end{align}
and
\begin{align}
    \mathrm{DTC}(\mathbf z^{[M]})
=
H(\mathbf z^{[M]})
-
\sum_{m=1}^{M}
H(\mathbf z^{(m)}| \mathbf z^{[M]\backslash(m)}).
\end{align}

The sum of One-vs-All mutual information terms can be rewritten as
\begin{align}
   \sum_{m=1}^{M}
I(\mathbf z^{(m)};\mathbf z^{[M]\backslash(m)})
=
\sum_{m=1}^{M}
\left[
H(\mathbf z^{(m)})
-
H(\mathbf z^{(m)}\mid \mathbf z^{[M]\backslash(m)})
\right]. 
\end{align}
Adding and subtracting \(H(\mathbf z^{[M]})\), we obtain
\begin{align}
    \sum_{m=1}^{M}
I(\mathbf z^{(m)};\mathbf z^{[M]\backslash(m)})
&= \sum_{m=1}^{M}
\left[
H(\mathbf z^{(m)}) - H(\mathbf z^{[M]}) + H(\mathbf z^{[M]})
-
H(\mathbf z^{(m)}\mid \mathbf z^{[M]\backslash(m)})
\right]\\
&=\mathrm{TC}(\mathbf z^{[M]})
+
\mathrm{DTC}(\mathbf z^{[M]}).
\end{align}
By Han's entropy inequalities \cite{HAN1978133}:
\begin{align}
    H(\mathbf{z}^{[M]}) \leq \frac{1}{M-1} \sum_{m=1}^M H(\mathbf{z}^{[M]\backslash(m)}),
\end{align}
total correlation and dual total correlation satisfy
\begin{align}
    \mathrm{TC}(\mathbf z^{[M]})
\leq
(M-1)\mathrm{DTC}(\mathbf z^{[M]}),
\label{eq:tc_dtc}
\end{align}
and
\begin{align}
   \mathrm{DTC}(\mathbf z^{[M]})
\leq
(M-1)\mathrm{TC}(\mathbf z^{[M]}).
\label{eq:dtc_tc}
\end{align}

Using Inequality (\ref{eq:tc_dtc}),
\begin{align}
    \frac{1}{M}
\sum_{m=1}^{M}
I(\mathbf z^{(m)};\mathbf z^{[M]\backslash(m)})
=
\frac{1}{M}
\left[
\mathrm{TC}(\mathbf z^{[M]})
+
\mathrm{DTC}(\mathbf z^{[M]})
\right]
\leq
\mathrm{DTC}(\mathbf z^{[M]}).
\label{eq: ineq 1}
\end{align}

Using Inequality (\ref{eq:dtc_tc}), 
\begin{align}
   \mathrm{DTC}(\mathbf z^{[M]}) &= \frac{1}{M}\mathrm{DTC}(\mathbf z^{[M]}) + \frac{M-1}{M}\mathrm{DTC}(\mathbf z^{[M]}) \notag\\
   &\leq \frac{M-1}{M}\mathrm{TC}(\mathbf z^{[M]}) + \frac{M-1}{M}\mathrm{DTC}(\mathbf z^{[M]})\notag \\
&=
\frac{M-1}{M}
\left[
\mathrm{TC}(\mathbf z^{[M]})
+
\mathrm{DTC}(\mathbf z^{[M]})
\right] \notag\\
&=
\frac{M-1}{M}
\sum_{m=1}^{M}
I(\mathbf z^{(m)};\mathbf z^{[M]\backslash(m)}).
\label{eq:ineq 2}
\end{align}

Combining Inequality (\ref{eq: ineq 1}) and Inequality (\ref{eq:ineq 2}) gives
\begin{align}
    \frac{1}{M}
\sum_{m=1}^{M}
I(\mathbf z^{(m)};\mathbf z^{[M]\backslash(m)})
\leq
\mathrm{DTC}(\mathbf z^{[M]})
\leq
\frac{M-1}{M}
\sum_{m=1}^{M}
I(\mathbf z^{(m)};\mathbf z^{[M]\backslash(m)}),
\end{align}
which completes the proof.
\end{proof}
\subsection{Proof of Theorem \ref{prop:m_kl_minimality}}\label{sec:proof_m_kl_minimality}
\begin{proof}
\begin{align}
    I(\mathbf{z}^{(m)}; \mathbf{x}^{(m)}) 
    &= \iint p_{\theta^{(m)}}(\mathbf{z}, \mathbf{x}^{(m)}) 
       \log \frac{p_{\theta^{(m)}}(\mathbf{z}|\mathbf{x}^{(m)})}{p_{\theta^{(m)}}(\mathbf{z})} 
       \, d\mathbf{z} \, d\mathbf{x}^{(m)} \label{eq:mi_def} \\
    &=   \int\underbrace{\int \dots \int}_{M}
       p_{\theta^{(m)}}(\mathbf{z}|\mathbf{x}^{(m)}) 
       p( \mathbf{x}^{(m)}, \mathbf{x}^{[M] \backslash (m)}) \nonumber \\
    &\quad \times \log \frac{p_{\theta^{(m)}}(\mathbf{z}|\mathbf{x}^{(m)})}{p_{\theta^{(m)}}(\mathbf{z})} 
       \, d\mathbf{z} d\mathbf{x}^{[M])}  \\
    &= \underbrace{\int \dots \int}_{M} \int 
       p_{\theta^{(m)}}(\mathbf{z}|\mathbf{x}^{(m)}) 
       p( \mathbf{x}^{(m)}, \mathbf{x}^{[M] \backslash (m)}) \nonumber \\
    &\quad \times \log \left( 
       \frac{p_{\theta^{(m)}}(\mathbf{z}|\mathbf{x}^{(m)})}{\prod_{i \neq m} p_{\theta^{(i)}}(\mathbf{z}|\mathbf{x}^{(i)})} 
       \cdot \frac{\prod_{i \neq m} p_{\theta^{(i)}}(\mathbf{z}|\mathbf{x}^{(i)})}{p_{\theta^{(m)}}(\mathbf{z})} 
       \right) d\mathbf{x}^{[M]}
       d\mathbf{z} ,
\end{align}
where $d\mathbf{x}^{[M]} =\, d\mathbf{x}^{(1)} \dots d\mathbf{x}^{(M)}$ for brevity, and  $p_{\theta^{(m)}}=\int p_{\theta^{(m)}}(\mathbf z|\mathbf{x}^{(m)})p(\mathbf{x}^{(m)})d\mathbf{x}^{(m)}$ is the marginal representation distribution induced by the $m$-th encoder, and $\int {\prod_{i \neq m} p_{\theta^{(i)}}(\mathbf{z}|\mathbf{x}^{(i)})} d\mathbf{z} \ne 1$, so $ {\prod_{i \neq m} p_{\theta^{(i)}}(\mathbf{z}|\mathbf{x}^{(i)})}$ cannot be regarded as a probability. Under this condition, we introduce the normalization constant:
\begin{align}
    C = \int {\prod_{i \neq m} p_{\theta^{(i)}}(\mathbf{u}|\mathbf{x}^{(i)})} d\mathbf{u},
\end{align}
and we denote by $p_{\theta^{[M]\backslash(m)}}(\mathbf{z}|\mathbf x^{[M]\backslash(m)}) = \frac{\prod_{i \neq m} p_{\theta^{(i)}}(\mathbf{z}|\mathbf{x}^{(i)})}{C}$. For the logarithm, we can have:
\begin{align}
    &\quad\ \log \left( 
       \frac{p_{\theta^{(m)}}(\mathbf{z}|\mathbf{x}^{(m)})}{\prod_{i \neq m} p_{\theta^{(i)}}(\mathbf{z}|\mathbf{x}^{(i)})} 
       \cdot \frac{\prod_{i \neq m} p_{\theta^{(i)}}(\mathbf{z}|\mathbf{x}^{(i)})}{p_{\theta^{(m)}}(\mathbf{z})} 
       \right) \\
       &= \log \left( 
       \frac{p_{\theta^{(m)}}(\mathbf{z}|\mathbf{x}^{(m)})}{\prod_{i \neq m} p_{\theta^{(i)}}(\mathbf{z}|\mathbf{x}^{(i)})} 
       \cdot \frac{\prod_{i \neq m} p_{\theta^{(i)}}(\mathbf{z}|\mathbf{x}^{(i)})}{p_{\theta^{(m)}}(\mathbf{z})}\cdot \frac{C}{C}  
       \right) \\  
       &=\log \left( 
       \frac{p_{\theta^{(m)}}(\mathbf{z}|\mathbf{x}^{(m)})}{\prod_{i \neq m} p_{\theta^{(i)}}(\mathbf{z}|\mathbf{x}^{(i)})/C} 
       \cdot \frac{\prod_{i \neq m} p_{\theta^{(i)}}(\mathbf{z}|\mathbf{x}^{(i)})/C}{p_{\theta^{(m)}}(\mathbf{z})} 
       \right) \\
       &=\log \left( 
       \frac{p_{\theta^{(m)}}(\mathbf{z}|\mathbf{x}^{(m)})}{p_{\theta^{[M]\backslash(m)}}(\mathbf{z}|\mathbf x^{[M]\backslash(m)})} 
       \cdot \frac{p_{\theta^{[M]\backslash(m)}}(\mathbf{z}|\mathbf x^{[M]\backslash(m)})}{p_{\theta^{(m)}}(\mathbf{z})} 
       \right) \\
       &=\log \left( 
       \frac{p_{\theta^{(m)}}(\mathbf{z}|\mathbf{x}^{(m)})}{p_{\theta^{[M]\backslash(m)}}(\mathbf{z}|\mathbf x^{[M]\backslash(m)})}\right) - \log \left(\frac{p_{\theta^{(m)}}(\mathbf{z})}{p_{\theta^{[M]\backslash(m)}}(\mathbf{z}|\mathbf x^{[M]\backslash(m)})}
       \right). 
\end{align}
Therefore, the mutual information becomes:
\begin{align}
 & \quad\  I(\mathbf{z}^{(m)}; \mathbf{x}^{(m)})\\ &= \underbrace{\int \dots \int}_{M}p(\mathbf{x}^{[M]}) \int 
        p_{\theta^{(m)}}(\mathbf{z}|\mathbf{x}^{(m)}) 
       \log \left( 
       \frac{p_{\theta^{(m)}}(\mathbf{z}|\mathbf{x}^{(m)})}{p_{\theta^{[M]\backslash(m)}}(\mathbf{z}|\mathbf x^{[M]\backslash(m)})}\right)d\mathbf{z} \, d\mathbf{x}^{[M]}\nonumber \\
    &\quad -\underbrace{\int \dots \int}_{M}  p(\mathbf{x}^{[M]})\int 
       p_{\theta^{(m)}}(\mathbf{z}|\mathbf{x}^{(m)}) 
      \log \left(\frac{p_{\theta^{(m)}}(\mathbf{z})}{p_{\theta^{[M]\backslash(m)}}(\mathbf{z}|\mathbf x^{[M]\backslash(m)})}
       \right)d\mathbf{z} \, d\mathbf{x}^{[M]} \\   
    &= \mathbb{E}_{p(\mathbf{x}^{[M]})} \left[ 
       D_{\mathrm{KL}} \left( p_{\theta^{(m)}}(\mathbf{z}|\mathbf{x}^{(m)}) ~\Big\|~ p_{\theta^{[M]\backslash(m)}}(\mathbf{z}|\mathbf x^{[M]\backslash(m)}) \right) \right] \nonumber \\
    &\quad -\underbrace{\int \dots \int}_{M}  \int p(\mathbf{x}^{[M]}) 
       p_{\theta^{(m)}}(\mathbf{z}|\mathbf{x}^{(m)}) 
      \log \left(\frac{p_{\theta^{(m)}}(\mathbf{z})}{p_{\theta^{[M]\backslash(m)}}(\mathbf{z}|\mathbf x^{[M]\backslash(m)})}
       \right)d\mathbf{z} \, d\mathbf{x}^{[M]},    
\end{align}

where $D_{\mathrm{KL}}(p(x)||q(x))=\int p(x) \log \frac{p(x)}{q(x)}dx \geq 0$. $p(\mathbf{x}^{[M]}) p_{\theta^{(m)}}(\mathbf{z}|\mathbf{x}^{(m)})$ defines a joint distribution over $(\mathbf{x}^{[M]}, \mathbf{z})$. Its marginal distribution over $\mathbf{z}$ is 
\begin{align}
    \int p(\mathbf{x}^{[M]}) p_{\theta^{(m)}}(\mathbf{z}|\mathbf{x}^{(m)})d\mathbf{x}^{[M]} = \int p(\mathbf{x}^{(m)}) p_{\theta^{(m)}}(\mathbf{z}|\mathbf{x}^{(m)}) = p_{\theta^{(m)}}(\mathbf{z}).
\end{align}
Therefore, by Bayes's rule, the induced conditional distribution over $\mathbf{x}^{[M]}$ given $\mathbf{z}$ is 
\begin{align}
    p_{\theta^{(m)}}(\mathbf{x}^{[M]}| \mathbf{z}) = \frac{p(\mathbf{x}^{[M]})p_{\theta^{(m)}}(\mathbf{z}|\mathbf{x}^{(m)})}{p_{\theta^{(m)}}(\mathbf{z})}.
\end{align}
Using this factorization, the mutual information can be rewritten as 
\begin{align}
    &\quad \ I(\mathbf{z}^{(m)}; \mathbf{x}^{(m)})\\ &= \mathbb{E}_{p( \mathbf{x}^{[M]})} \left[ 
       D_{\mathrm{KL}} \left( p_{\theta^{(m)}}(\mathbf{z}|\mathbf{x}^{(m)}) ~\Big\|~ p_{\theta^{[M]\backslash(m)}}(\mathbf{z}|\mathbf x^{[M]\backslash(m)}) \right) \right] \nonumber \\
    &\quad -\underbrace{\int \dots \int}_{M}  \int p_{\theta^{(m)}}(\mathbf{x}^{[M]}| \mathbf{z})p_{\theta^{(m)}}(\mathbf{z})
      \log \left(\frac{p_{\theta^{(m)}}(\mathbf{z})}{p_{\theta^{[M]\backslash(m)}}(\mathbf{z}|\mathbf x^{[M]\backslash(m)})}
       \right)d\mathbf{z} \, d\mathbf{x}^{[M]} \\ 
    &= \mathbb{E}_{p( \mathbf{x}^{[M]})} \left[ 
       D_{\mathrm{KL}} \left( p_{\theta^{(m)}}(\mathbf{z}|\mathbf{x}^{(m)}) ~\Big\|~ p_{\theta^{[M]\backslash(m)}}(\mathbf{z}|\mathbf x^{[M]\backslash(m)}) \right) \right] \nonumber \\
    &\quad -\underbrace{\int \dots \int}_{M} p_{\theta^{(m)}}(\mathbf{x}^{[M]}| \mathbf{z}) \int p_{\theta^{(m)}}(\mathbf{z})
      \log \left(\frac{p_{\theta^{(m)}}(\mathbf{z})}{p_{\theta^{[M]\backslash(m)}}(\mathbf{z}|\mathbf x^{[M]\backslash(m)})}
       \right)d\mathbf{z} \, d\mathbf{x}^{[M]} \\
    &= \mathbb{E}_{p(\mathbf{x}^{[M]})} \left[ 
       D_{\mathrm{KL}} \left( p_{\theta^{(m)}}(\mathbf{z}|\mathbf{x}^{(m)}) ~\Big\|~ p_{\theta^{[M]\backslash(m)}}(\mathbf{z}|\mathbf x^{[M]\backslash(m)}) \right) \right] \nonumber \\
    &\quad - \mathbb{E}_{p( \mathbf{x}^{[M]} | \mathbf{z})} \left[ 
       D_{\mathrm{KL}} \left( p_{\theta^{(m)}}(\mathbf{z}) ~\Big\|~ p_{\theta^{[M]\backslash(m)}}(\mathbf{z}|\mathbf x^{[M]\backslash(m)}) \right) \right] \\
    &\leq \mathbb{E}_{p( \mathbf{x}^{[M]})} \left[ 
       D_{\mathrm{KL}} \left( p_{\theta^{(m)}}(\mathbf{z}|\mathbf{x}^{(m)}) ~\Big\|~ p_{\theta^{[M]\backslash(m)}}(\mathbf{z}|\mathbf x^{[M]\backslash(m)}) \right) \right]. \label{eq:mi_final}
\end{align}
\end{proof}

\subsection{Proof of Proposition \ref{prop:gaussain_m_minimality}}
\label{sec:proof_gaussian_m_minimality}
\begin{proof}[Proof of Proposition \ref{prop:gaussain_m_minimality}]
Let $p_{\theta^{(m)}}(\mathbf{z}|\mathbf{x}^{(m)})=\mathcal{N}(\mathbf{z}; \mu^{(m)}, \sigma^2I)$ with $\mu^{(m)}=\mu_{\theta^{(m)}}(\mathbf{x}^{(m)})$, so $p_{\theta^{[M]\backslash(m)}}(\mathbf{z}|\mathbf x^{[M]\backslash(m)})$ also follows a Gaussian distribution, namely:
\begin{align}
    p_{\theta^{[M]\backslash(m)}}(\mathbf{z}|\mathbf x^{[M]\backslash(m)}) \sim \mathcal{N}(\mathbf{z;\bar{\mu}}, \frac{\sigma^2}{M-1}I),
\end{align}
where the mean $\bar{\mu}=\frac{1}{M-1}\sum_{i=1, i\ne m}^M \mu^{(i)}$ and the covariance $\frac{\sigma^2}{M-1}I$. Using the closed-form expression of the KL divergence between Gaussian distributions, we obtain:
  \begin{align}
      D_{\mathrm{KL}} \left( p_{\theta^{(m)}}(\mathbf{z}|\mathbf{x}^{(m)}) ~\Big\|~ p_{\theta^{[M]\backslash(m)}}(\mathbf{z}|\mathbf x^{[M]\backslash(m)}) \right) = \frac{M-1}{2\sigma^2}||\mu^{(m)}-\bar{\mu}||_2^2 + C_M,
  \end{align}
  where $C_M$ is a constant independent of the mean parameters. Therefore, up to constants, minimizing this KL divergence is equivalent to minimizing the squared distance between $\mu^{(m)}$ and the average of the means of the rest of the modalities $\bar{\mu}$, namely:
  \begin{align}
      D_{\mathrm{KL}} \left( p_{\theta^{(m)}}(\mathbf{z}|\mathbf{x}^{(m)}) ~\Big\|~ p_{\theta^{[M]\backslash(m)}}(\mathbf{z}|\mathbf x^{[M]\backslash(m)}) \right) \propto ||\mu_{\theta^{(m)}}(\mathbf{x}^{(m)})-\bar{\mu}||_2^2.
  \end{align}
Then
\begin{align}
&\mathbb{E}_{p(\mathbf{x}^{(1)},\dots,\mathbf{x}^{(M)})}\left[D_{\mathrm{KL}} \left( p_{\theta^{(m)}}(\mathbf{z}|\mathbf{x}^{(m)}) ~\Big\|~ p_{\theta^{[M]\backslash(m)}}(\mathbf{z}|\mathbf x^{[M]\backslash(m)})\right)\right]  \\\propto & \ \mathbb{E}_{p(\mathbf{x}^{(1)},\dots,\mathbf{x}^{(M)})}\left[||\mu_{\theta^{(m)}}(\mathbf{x}^{(m)})-\bar{\mu}||_2^2\right].
  \end{align}
\end{proof}

\section{Experimental Details}
\subsection{Dataset}
\label{sec:dataset_detail}
\textbf{Vision\&Touch.}
We use the Vision\&Touch dataset~\citep{lee2019making}, which provides synchronized modalities from multiple sensors, including third-person RGB images, force--torque histories, end-effector pose, depth, and optical flow. Using this dataset, we construct the following downstream tasks:

\emph{Vision--Force--Depth (\textsc{VFD}).}
This setting uses RGB images \([3 \times 128 \times 128]\), force--torque histories \([32 \times 6]\), and depth observations \([1 \times 128 \times 128]\). We evaluate two tasks: next-step end-effector orientation prediction, formulated as a 4-dimensional regression problem, and trajectory-pair classification, where the model predicts whether two modality triplets are sampled from the same trajectory.

\emph{Vision--Force--Proprioception (\textsc{VFP}).}
This setting uses RGB images, force--torque histories, and end-effector pose vectors \([7]\). The downstream task is binary contact prediction, where the model classifies whether the end-effector is in contact with the object.

\emph{Higher-modality trajectory classification.}
To evaluate OVA-IB beyond the standard three-modality setting, we construct four- and five-modality trajectory classification tasks from Vision \& Touch. The four-modality (VFPD) setting uses RGB, force--torque, proprioception, and depth, while the five-modality (VFPDO) setting additionally includes optical flow \([2 \times 128 \times 128]\). In both settings, the model predicts whether two modality tuples are sampled from the same trajectory.

\textbf{MuJoCo Push.}
We use the MuJoCo Push dataset \cite{10.1109/IROS45743.2020.9341579}, a planar pushing benchmark in which a Franka Emika Panda arm interacts with a puck. The modalities include grayscale image histories \([32 \times 32 \times 32]\), force--torque histories \([32 \times 6]\), and end-effector pose histories \([32 \times 7]\), where the leading dimension corresponds to temporal history. The task is to predict the next-step 2D position of the object on the table.

\textbf{CMU-MOSEI.}
We use the CMU-MOSEI benchmark preprocessed by MultiBench~\citep{liang2021multibench}, which contains visual \([35 \times 713]\), acoustic \([32 \times 74]\), and textual \([32 \times 300]\) modalities. Following the standard binary sentiment setting, the task is to classify whether the sentiment label is positive.

\textbf{WESAD.} We use WESAD~\cite{10.1145/3242969.3242985} to evaluate both cross-device physiological retrieval and downstream affective state classification. The dataset provides synchronized recordings from chest- and wrist-worn wearable devices. For retrieval, we construct the task \(\text{Resp}+\text{ECG}\rightarrow\text{BVP}\), where chest respiration ($[1\times3500]$) and ECG ($[1\times3500]$) retrieve the synchronized wrist BVP signal ($[1\times320]$). This task is physiologically grounded: ECG reflects cardiac electrical activity, respiration provides cardiopulmonary context, and BVP captures peripheral pulse dynamics at the wrist. It thus evaluates whether learned representations align central chest-side signals with peripheral wrist-side responses from the same time window. Additionally, we evaluate the learned representations on a downstream three-class affective classification task (baseline, stress, amusement).
\subsection{Experimental Implementation}
\label{sec:experimental_implementation}
We implement all methods using modality-specific encoders followed by projection heads. To ensure a fair comparison, OVA-IB and all baselines use the same backbone architectures, projection heads, downstream predictors, optimizer, batch size, temperature, and random seed whenever applicable. The only difference across methods is the pretraining objective.

\paragraph{Modality encoders.} For image-like modalities, we use ResNet-18 backbones \cite{he2015deepresiduallearningimage}. RGB inputs are processed by a standard ResNet-18 visual encoder. For depth observations, which are single-channel inputs, we duplicate the depth map along the channel dimension to form a 3-channel input and process it with a standard ResNet-18. For optical flow, we adapt the first convolutional layer of ResNet-18 to accept 2 input channels. For MuJoCo Push, the image history is treated as a 32-channel image input, and we adapt the first convolutional layer of ResNet-18 to accept 32 input channels while keeping the remaining architecture unchanged.

For force--torque histories, we use a one-dimensional convolutional encoder. The force history sequence is passed through 1D convolutional layers with ReLU activations, followed by adaptive average pooling and a final linear layer. This extracts temporal sensor patterns while producing a fixed-dimensional representation independent of the input sequence length. For low-dimensional proprioceptive inputs such as end-effector pose vectors, we use a lightweight MLP encoder with LeakyReLU activation. For CMU-MOSEI, the visual, acoustic, and textual modalities are provided as pre-extracted temporal feature sequences, and we use LSTM \cite{6795963} encoders to aggregate temporal information for each modality. For WESAD, we use a 3-layer 1D-CNN as the encoder.

All modality encoders output a 256-dimensional representation. Each representation is then passed through a modality-specific projection head using a 3-layer MLP, which maps it to a 256-dimensional embedding used for contrastive alignment and the OVA-IB objective.

\paragraph{Hyperparameters.} All models are optimized with Adam. We use a batch size of 64, temperature $\tau=0.01$, representation dimension 256, and embedding dimension 256 for all experiments. Learning rates are selected by dataset and downstream task: $10^{-3}$ for MuJoCo Push regression, $10^{-5}$ for CMU-MOSEI sentiment classification, $10^{-4}$ for VFD trajectory classification, $10^{-3}$ for VFD orientation regression, and $10^{-4}$ for the four- and five-modality VFPD/VFPDO trajectory classification tasks. For all tasks, we use random seed 42, and we also use 0 and 1 for the cross-modal retrieval. For our OVA-IB loss, we set $\lambda = 10^{-8}$.

\paragraph{Downstream evaluation.} After pretraining, we use the learned modality encoders to extract modality representations and concatenate the available modality embeddings for downstream prediction. The downstream predictor is a 4-layer MLP. For classification tasks, the predictor is trained with cross-entropy loss; for regression tasks, it is trained with mean squared error. For modality-agnostic evaluation, we reuse the encoders pre-trained by all modalities and evaluate them under different available modality combinations at test time.

\paragraph{Experimental Computer Resources}
Our experiments were run on an \textit{NVIDIA RTX 4090 GPU} with \textit{24 GB} memory and an \textit{Intel(R) Xeon(R) Platinum 8470Q CPU}.
\subsection{Modality-Agnostic Evaluation on Additional Tasks}
\label{sec:appendix-modality-agnostic experiment}
We report the full modality-combination evaluation results on the remaining benchmarks. 
Following Section \ref{sec:modality_agnostic_evaluation}, models are pretrained with all modalities and evaluated using different available modality combinations at test time. 
These tables (\ref{tab:mujoco_push}, \ref{tab:mosei_results}, \ref{tab:vfp_classification}, \ref{tab:vfd_error}, \ref{tab:four_modality_acc}, \ref{tab:five_modality_acc}) complement the main VFD classification results in Table \ref{tab:vfd_classification} and show that OVA-IB remains robust across both single-modality and multi-modality evaluation settings.

\begin{table}[H]
  \centering
  \caption{Modality-combination evaluation on MuJoCo Push regression.}
  \label{tab:mujoco_push}
 
  \begin{tabular}{lccccccc}
    \toprule
    Method & V & F & P & V+F & V+P & F+P & V+F+P \\
    \midrule
    Ours               & 0.0393 & 0.7107 & 0.3803 & 0.0707 & \textbf{0.0348} & 0.3753 & \textbf{0.0340} \\
    Symile             & 0.1584 & 0.6562 & 0.3808 & 0.2599 & 0.0910 & 0.3818 & 0.1228 \\
    Gram               & \textbf{0.0314} & 0.6677 & 0.3793 & 0.0722 & 0.0296 & \textbf{0.3859} & 0.0425 \\
    TRIANGLE           & 0.0459 & \textbf{0.7351} & 0.3795 & 0.0692 & 0.0379 & 0.3794 & 0.0656 \\
    Pairwise CLIP      & 0.0410 & 0.7084 & 0.3818 & \textbf{0.0686} & 0.0393 & 0.3744 & 0.0810 \\
    Pairwise CLIP + IB & 0.1290 & 0.7031 & 0.3803 & 0.1839 & 0.1775 & 0.3763 & 0.3385 \\
    \bottomrule
  \end{tabular}
\end{table}

\begin{table}[H]
  \centering
  \caption{Modality-combination evaluation on CMU-MOSEI sentiment classification.}
  \label{tab:mosei_results}
  
  \begin{tabular}{lccccccc}
    \toprule
    Method & V & A & T & V+A & V+T & A+T & V+A+T \\
    \midrule
    Ours               & 71.30 & 71.04 & 75.76 & 71.36 & 75.65 & 75.76 & 76.02 \\
    Symile             & 71.04 & 71.04 & 72.99 & 71.04 & 71.04 & 71.04 & \textbf{77.00} \\
    Gram               & \textbf{71.60} & 71.04 & 76.28 & 71.56 & 76.64 & 75.85 & 72.77 \\
    TRIANGLE           & 71.04 & 71.04 & \textbf{77.01} & 71.04 & 75.78 & \textbf{76.56} & 76.13 \\
    Pairwise CLIP      & 71.23 & \textbf{71.08} & 76.10 & \textbf{71.73} & 76.08 & 76.04 & 75.05 \\
    Pairwise CLIP + IB & 71.04 & 71.04 & 76.19 & 71.04 & \textbf{76.92} & 76.06 & 76.25 \\
    \bottomrule
  \end{tabular}
\end{table}

\begin{table}[H]
  \centering
  \caption{Modality-combination evaluation on VFP contact classification.}
  \label{tab:vfp_classification}
  \begin{tabular}{lccccccc}
    \toprule
    Method & V & F & P & V+F & V+P & F+P & V+F+P \\
    \midrule
    Ours               & 82.77 & 96.74 & \textbf{83.97} & 96.71 & \textbf{83.56} & 96.99 & \textbf{97.12} \\
    Symile             & \textbf{83.32} & \textbf{97.42} & 82.90 & 95.71 & 83.27 & 97.33 & 96.44 \\
    Gram               & 83.23 & 97.10 & 83.22 & 95.66 & 83.22 & 97.40 & 96.85 \\
    TRIANGLE           & 82.31 & 97.16 & 83.08 & 95.81 & 82.19 & 97.06 & 95.91 \\
    Pairwise CLIP      & 83.01 & 97.24 & 83.22 & 95.68 & 82.98 & 97.19 & 96.55 \\
    Pairwise CLIP + IB & 82.57 & 97.28 & 83.19 & \textbf{96.76} & 82.80 & \textbf{97.53} & 97.10 \\
    \bottomrule
  \end{tabular}
\end{table}

\begin{table}[H]
  \centering
  \caption{Modality-combination evaluation on VFD orientation regression.}
  \label{tab:vfd_error}
  \begin{tabular}{lccccccc}
    \toprule
    Method & V & F & D & V+F & V+D & F+D & V+F+D \\
    \midrule
    Ours               & \textbf{0.0002} & 0.0022 & \textbf{0.0003} & \textbf{0.0005} & \textbf{0.0002} & \textbf{0.0006} & \textbf{0.0003} \\
    Symile             & 0.0017 & 0.0022 & 0.0017 & 0.0017 & 0.0016 & 0.0017 & 0.0015 \\
    Gram               & 0.0020 & 0.0022 & 0.0023 & 0.0019 & 0.0018 & 0.0022 & 0.0018 \\
    TRIANGLE           & 0.0019 & 0.0022 & 0.0021 & 0.0020 & 0.0018 & 0.0021 & 0.0019 \\
    Pairwise CLIP      & 0.0017 & \textbf{0.0020} & 0.0017 & 0.0018 & 0.0014 & 0.0017 & 0.0013 \\
    Pairwise CLIP + IB & 0.0010 & 0.0022 & 0.0015 & 0.0012 & 0.0008 & 0.0016 & 0.0009 \\
    \bottomrule
  \end{tabular}
\end{table}

\begin{table}[H]
  \centering
  \caption{Modality-combination evaluation on four-modality VFPD trajectory classification.}
  \label{tab:four_modality_acc}
  \begin{tabular}{lccccc}
    \toprule
    Method & V & F & P & D & V+F+P+D \\
    \midrule
    Ours               & 91.14 & \textbf{94.56} & \textbf{89.38} & 90.20 & \textbf{94.73} \\
    Symile             & 83.29 & 93.01 & 88.72 & 82.91 & 84.97 \\
    Gram               & 90.14 & 89.93 & 87.67 & 89.70 & 90.81 \\
    Pairwise CLIP      & 92.94 & 91.54 & 88.55 & 91.13 & 93.25 \\
    Pairwise CLIP + IB & \textbf{93.19} & 91.71 & 88.66 & \textbf{91.67} & 94.45 \\
    \bottomrule
  \end{tabular}
\end{table}

\begin{table}[H]
  \centering
  \caption{Modality-combination evaluation on five-modality VFPDO trajectory classification.}
  \label{tab:five_modality_acc}
  \begin{tabular}{lcccccc}
    \toprule
    Method & V & F & P & D & O & V+F+P+D+O \\
    \midrule
    Ours               & 90.79 & \textbf{95.37} & \textbf{90.81} & 90.02 & 87.14 & \textbf{95.91} \\
    Symile             & 82.18 & 92.70 & 87.88 & 82.55 & 82.42 & 82.96 \\
    Gram               & 90.40 & 89.11 & 87.50 & 89.99 & 81.62 & 90.59 \\
    Pairwise CLIP      & 92.54 & 92.06 & 87.89 & 89.27 & 86.28 & 92.95 \\
    Pairwise CLIP + IB & \textbf{93.27}
 & 92.02 & 89.06 & \textbf{91.33} & \textbf{88.26}
 & 94.68 \\
    \bottomrule
  \end{tabular}
\end{table}

\begin{table}[H]
  \centering
  \caption{Modality-combination evaluation on the WESAD dataset.}
  \label{tab:wesad_modality_agnostic}
  \begin{tabular}{lccccccc}
    \toprule
    Method & ECG & Resp & BVP & E+R & E+B & R+B & ERB \\
    \midrule
    Ours               & 75.06 & 67.29 & 61.49 & 76.52 & 75.96 & 69.43 & \textbf{76.18} \\
    Symile             & 64.58 & 66.50 & 64.75 & 72.80 & 72.80 & 70.33 & 73.42 \\
    Gram               & 53.21 & \textbf{68.24} & 65.15 & 67.96 & 62.61 & 61.99 & 66.22 \\
    TRIANGLE           & 65.37 & 68.02 & 65.43 & 57.26 & 60.42 & 69.54 & 56.87 \\
    Pairwise CLIP      & 61.37 & 65.65 & 65.37 & 73.48 & 72.75 & 69.54 & 70.83 \\
    Pairwise CLIP + IB & \textbf{75.39} & 66.39 & \textbf{67.29} & \textbf{77.14} & \textbf{76.97} & \textbf{70.78} & 73.03 \\
    \bottomrule
  \end{tabular}
\end{table}

\end{document}